%% file: neurips_main.tex
\documentclass{article}

\usepackage[nonatbib,preprint]{neurips_2020}
\usepackage[utf8]{inputenc} 
\usepackage[T1]{fontenc}    
\usepackage{hyperref}       
\usepackage{url}            
\usepackage{booktabs}       
\usepackage{amsfonts}       
\usepackage{nicefrac}       
\usepackage{microtype}      
\usepackage{abbrevs}
\usepackage{amssymb}
\usepackage{color}
\usepackage{epstopdf}
\usepackage{nicefrac}
\usepackage[english]{babel}
\usepackage{amsmath,accents}
\usepackage[colorinlistoftodos, disable]{todonotes}
\usepackage{amsthm}
\usepackage{multicol}
\usepackage{verbatim}
\usepackage[linesnumbered, ruled,vlined]{algorithm2e}
\usepackage{algorithm,algorithmic}

\usepackage{subfigure}

\input{defs}

\title{Stochastic Bandits with Linear Constraints}

\author{%
  Aldo Pacchiano \\
  UC Berkeley\\
  \texttt{pacchiano@berkeley.edu} \\
   \And
  Mohammad Ghavamzadeh \\
Google Research\\
  \texttt{ghavamza@google.com}\\
  \AND
  Peter Bartlett\\
  UC Berkeley\\
  \texttt{peter@berkeley.edu}\\
  \And
  Heinrich Jiang\\
  Google Research\\
  \texttt{heinrichj@google.com}\\
}

\begin{document}

\maketitle

\begin{abstract}
\input{abstract}
\end{abstract}

\input{introduction.tex}
\input{linear_bandits}

\input{multi_armed_bandits}
\input{experiments}

\input{previous_work.tex}



\bibliography{constrained_bandits}
\bibliographystyle{plainnat}

\newpage
\appendix
\input{appendix_linear}

\newpage

\input{contrained_muli_armed_bandit_appendix.tex}

\input{optimism_pessimism_MAB_appendix}

\input{LP_structure_appendix}

\input{regret_analysis_MAB_appendix}

\end{document}

%% file: defs.tex
\usepackage{url}
\usepackage{graphicx}
\usepackage[font=scriptsize]{caption}
\usepackage{amsmath}

\usepackage{amsfonts}
\usepackage{graphicx} 
\usepackage{natbib}

\usepackage{xcolor}
\usepackage{tabularx} 
\usepackage{amsthm}
\usepackage{amsfonts}
\usepackage{amsmath}
\usepackage{amssymb}
\usepackage{mathrsfs}
\usepackage{multirow}
\usepackage{bbm}
\usepackage{wrapfig}

\newtheorem{theorem}{Theorem}

\newtheorem{remark}{Remark}

\newtheorem{proposition}{Proposition}
\newtheorem{lemma}{Lemma}
\newcounter{assump}
\newtheorem{assumption}[assump]{Assumption}

\DeclareMathOperator*{\argmax}{arg\,max}

\usepackage{enumitem}
\usepackage{wrapfig}

\newcommand{\beq}{\begin{equation}}

\newcommand{\eeq}{\end{equation}}
\newcommand{\beqs}{\begin{equation*}}
\newcommand{\eeqs}{\end{equation*}}

\usepackage[framemethod=TikZ]{mdframed}
\usepackage{amsthm}
\newcounter{theo}[section] 
\renewcommand{\thetheo}{\arabic{section}.\arabic{theo}}

\newcounter{lem}[section] 
\renewcommand{\thelem}{\arabic{section}.\arabic{lem}}

\newcounter{prf}[section]
\renewcommand{\theprf}{\arabic{section}.\arabic{prf}}

\newcounter{cor}[section]
\renewcommand{\thecor}{\arabic{section}.\arabic{cor}}

\newcounter{prop}[section] 
\renewcommand{\thelem}{\arabic{section}.\arabic{prop}}

%% file: abstract.tex
We study a constrained contextual linear bandit setting, where the goal of the agent is to produce a sequence of policies, whose expected cumulative reward over the course of $T$ rounds is maximum, and each has an expected cost below a certain threshold $\tau$. We propose an upper-confidence bound algorithm for this problem, called {\em optimistic pessimistic linear bandit} (OPLB), and prove an $\widetilde{\mathcal{O}}(\frac{d\sqrt{T}}{\tau-c_0})$ bound on its $T$-round regret, where the denominator is the difference between the constraint threshold and the cost of a known feasible action. We further specialize our results to multi-armed bandits and 
propose a computationally efficient algorithm for this setting. We prove a regret bound of $\widetilde{\mathcal{O}}(\frac{\sqrt{KT}}{\tau - c_0})$ for this algorithm in $K$-armed bandits, which is a $\sqrt{K}$ improvement over the regret bound we obtain by simply casting multi-armed bandits as an instance of contextual linear bandits and using the regret bound of OPLB. We also prove a lower-bound for the problem studied in the paper and provide simulations to validate our theoretical results.  





%% file: introduction.tex
\section{Introduction}
\label{sec:intro}

A {\em multi-armed bandit} (MAB)~\citep{lai85asymptotically,auer02finitetime,lattimore2018bandit} is an online learning problem in which the agent acts by pulling arms. After an arm is pulled, the agent receives its {\em stochastic reward}. The goal of the agent is to maximize its expected cumulative reward without knowledge of the arms' distributions. To achieve this goal, the agent has to balance its {\em exploration} and {\em exploitation}: to decide when to {\em explore} and learn about the arms, and when to {\em exploit} and pull the arm with the highest estimated reward thus far. A \emph{stochastic linear bandit}~\citep{dani08stochastic,rusmevichientong10linearly,abbasi2011improved} is a generalization of MAB to the setting where each of (possibly) infinitely many arms is associated with a feature vector. The mean reward of an arm is the dot product of its feature vector and an unknown parameter vector, which is shared by all the arms. This formulation contains time-varying action (arm) sets and feature vectors, and thus, includes the {\em linear contextual bandit} setting. These models capture many practical applications spanning clinical trials~\citep{Villar15MA}, recommendation systems~\citep{Li10CB,balakrishnan2018using}, wireless networks~\citep{maghsudi2016multi}, sensors~\citep{washburn2008application}, and strategy games \citep{ontanon2013combinatorial}. 
The most popular exploration strategies in stochastic bandits are {\em optimism in the face of uncertainty} (OFU)~\citep{auer02finitetime} and {\em Thompson sampling} (TS)~\citep{thompson33likelihood,agrawal13further,russo18tutorial} that are relatively well understood in both multi-armed and linear bandits~\citep{dani08stochastic,abbasi2011improved,agrawal13thompson,lattimore2018bandit}.

In many practical problems, the agent requires to satisfy certain operational constraints while maximizing its cumulative reward. Depending on the form of the constraints, several {\em constrained stochastic bandit} settings have been formulated and analyzed. One such setting is what is known as {\em knapsack bandits}. In this setting, pulling each arm, in addition to producing a reward signal, results in a random consumption of a global budget, and the goal is to maximize the cumulative reward before the budget is fully consumed (e.g.,~\citealt{badanidiyuru2013bandits,badanidiyuru2014resourceful,agrawal2014bandits,wu2015algorithms,agrawal2016linear}). Another such setting is referred to as {\em conservative bandits}. In this setting, there is a baseline arm or policy, and the agent, in addition to maximizing its cumulative reward, should ensure that at each round, the difference between its cumulative reward and that of the baseline remains below a predefined fraction of the baseline cumulative reward~\citep{wu2016conservative,kazerouni2017conservative,Garcelon20IA}. In these two settings, the constraint applies to a cumulative quantity (budget consumption or reward) over the entire run of the algorithm. Thus, the set of feasible actions at each round is a function of the history of the algorithm. 

Another constrained bandit setting is where each arm is associated with two (unknown) distributions, generating reward and cost signals. The goal is to maximize the cumulative reward, while making sure that with high probability, the expected cost of the arm pulled at each round is below a certain threshold. Here the constraint is stage-wise, and unlike the last two settings, is independent of the history.~\citet{amani2019linear} and~\citet{Moradipari19SL} have recently studied this setting for linear bandits and derived and analyzed explore-exploit~\citep{amani2019linear} and Thompson sampling~\citep{Moradipari19SL} algorithms for it. 

This setting is the closest to the one we study in this paper. In our setting, we also assume two distributions for each arm, one for reward and for cost. At each round the agent constructs a policy according to which it takes its action. The goal of the agent is to produce a sequence of policies with maximum expected cumulative reward, while making sure that the expected cost of the constructed policy (not the pulled arm) at each round is below a certain threshold. This is a linear constraint and can be easily extended to more constraints by having more cost distributions associated to each arm, one per each constraint. Compared to the previous setting, our constraint is more relaxed (from high-probability to expectation), and as a result, it would be possible for us to obtain a solution with larger expected cumulative reward. We will have a detailed discussion on the relationship between these two settings and the similarities and differences of our results with those reported in~\citet{amani2019linear} and~\citet{Moradipari19SL} in Section~\ref{sec:related-work}.

In this paper, we study the above setting for contextual linear bandits. After defining the setting in Section~\ref{sec:setting}, we propose an upper-confidence bound (UCB) algorithm for it, called {\em optimistic pessimistic linear bandit} (OPLB), in Section~\ref{sec:algo}. We prove an $\widetilde{\mathcal{O}}(\frac{d\sqrt{T}}{\tau-c_0})$ bound on the $T$-round regret of OPLB in Section~\ref{section::lin_opt_pess_analysis}, where $d$ is the action dimension and $\tau-c_0$ is the difference between the constraint threshold and the cost of a known feasible action. The action set considered in our contextual linear bandit setting is general enough to include MAB. However, in Section~\ref{sec:constrained-MAB}, we further specialize our results to MAB and propose a computationally efficient algorithm for this setting, called {\em optimistic pessimistic bandit} (OPB). We show that in the MAB case, there always exists a feasible optimal policy with probability mass on at most $m+1$ arms, where $m$ is the number of linear constraints. This property plays an important role in the computational efficiency of OPB. We prove a regret bound of $\widetilde{\mathcal{O}}(\frac{\sqrt{KT}}{\tau - c_0})$ for OPB in $K$-armed bandits, which is a $\sqrt{K}$ improvement over the regret bound we obtain by simply casting MAB as an instance of contextual linear bandit and using the regret bound of OPLB. We also prove a lower-bound for the problem studied in the paper and provide simulations to validate our theoretical results.

%% file: linear_bandits.tex
\section{Problem Formulation}
\label{sec:setting}

We adopt the following notation. The set $\{1,\ldots,T\}$ is denoted by $[T]$. We represent the set of distributions with support over a compact set $\mathcal S$ by $\Delta_{\mathcal S}$. We denote by $\langle x,y \rangle:=x^\top y\in\mathbb R$, the inner product of two vectors $x,y\in\mathbb R^d$, and by $\|x\|:=\sqrt{x^\top x}$, the $\ell_2$-norm of vector $x$. 

The setting we study in this paper is {\em contextual linear bandit} with {\em linear constraints}. In each round $t$, the agent is given an decision set $\mathcal A_t\subset{\mathbb R}^d$ from which it has to choose an action $x_t$. Upon taking action $x_t\in\mathcal A$, it observes a pair $(r_t,c_t)$, where $r_t=\langle x_t,\theta_*\rangle + \xi^r_t$ and $c_t=\langle x_t,\mu_*\rangle + \xi^c_t$ are the reward and cost signals, respectively. In the reward and cost definitions, $\theta_*\in{\mathbb R}^d$ and $\mu_*\in{\mathbb R}^d$ are the unknown {\em reward} and {\em cost parameters}, and $\xi^r_t$ and $\xi^c_t$ are reward and cost noise, satisfying conditions that will be specified soon. The agent selects its action $x_t\in\mathcal A_t$ in each round $t$ according to its policy $\pi_t\in\Delta_{\mathcal A_t}$ at that round, i.e.,~$x_t\sim\pi_t$. 

The goal of the agent is to produce a sequence of policies $\{\pi_t\}_{t=1}^T$ with maximum {\em expected cumulative reward} over the course of $T$ rounds, while satisfying the {\bf\em linear constraint} 
\begin{equation}
\label{eq:constraint}
\mathbb E_{x\sim\pi_t}[\langle x,\mu_* \rangle]\leq \tau,\quad\forall t\in[T], \qquad (\tau\geq 0 \;\; \text{is referred to as the \em constraint threshold}). 
\end{equation}
Thus, the policy $\pi_t$ selected by the agent in each round $t\in [T]$ should belong to the set of {\em feasible policies} over the action set $\mathcal A_t$, i.e.,~$\Pi_t=\{\pi\in\Delta_{\mathcal A_t}:\mathbb E_{x\sim\pi}[\langle x,\mu_* \rangle]\leq \tau\}$. Maximizing the expected cumulative reward in $T$ rounds is equivalent to minimizing the $T$-round {\em constrained pseudo-regret},\footnote{In the rest of the paper, we simply refer to the $T$-round constrained pseudo-regret $\mathcal R_\Pi(T)$ as $T$-round regret.}
%
\begin{equation}
\label{eq:regret}
\mathcal R_\Pi(T) = \sum_{t=1}^T \mathbb E_{x\sim\pi_t^*}[\langle x,\theta_* \rangle] - E_{x\sim\pi_t}[\langle x,\theta_* \rangle], 
\end{equation}
%
where $\pi_t,\pi^*_t\in\Pi_t\;\forall t\in[T]$ and $\pi^*_t$ is the {\em optimal feasible} policy at round $t$, i.e.,~$\pi^*_t\in\max_{\pi\in\Pi_t}\mathbb E_{x\sim\pi}[\langle x,\theta_* \rangle]$. The terms $\mathbb E_{x\sim\pi}[\langle x,\theta_* \rangle]$ and $\mathbb E_{x\sim\pi}[\langle x,\mu_* \rangle]$ in~\eqref{eq:constraint} and~\eqref{eq:regret} are the expected reward and cost of policy $\pi$, respectively. Thus, a feasible policy is the one whose expected cost is below the constraint threshold $\tau$, and the optimal feasible policy is a feasible policy with maximum expected reward. We use the shorthand notations $x_\pi := \mathbb{E}_{x \sim \pi}[x]$, $r_\pi:=\mathbb E_{x\sim\pi}[\langle x,\theta_* \rangle]$ and $c_\pi:=\mathbb E_{x\sim\pi}[\langle x,\mu_* \rangle]$ for the expected action, reward and cost of a policy $\pi$. With these shorthand notations, we may write the $T$-round pseudo-regret as $\mathcal R_\Pi(T)=\sum_{t=1}^T r_{\pi^*_t} - r_{\pi_t}$.

We make the following assumptions for our setting. The first four assumptions are standard in linear bandits. The fifth one is necessary to guarantee constraint satisfaction ({\em safety}). 

\begin{assumption}
\label{ass:noise-sub-gaussian}
For all $t\in[T]$, the reward and cost noise random variables $\xi_t^r$ and $\xi_t^c$ are conditionally $R$-sub-Gaussian, i.e.,
%
\begin{align*}
&\mathbb{E}[\xi_t^r \mid \mathcal{F}_{t-1}] = 0, \qquad \mathbb{E}[\exp(\alpha \xi_t^r) \mid \mathcal{F}_{t-1}] \leq \exp(\alpha^2 R^2/2),\;\;\;\forall \alpha\in\mathbb R, \\
&\mathbb{E}[\xi_t^c \mid \mathcal{F}_{t-1}] = 0, \qquad \mathbb{E}[\exp(\alpha \xi_t^c) \mid \mathcal{F}_{t-1}] \leq \exp(\alpha^2 R^2/2),\;\;\;\forall \alpha\in\mathbb R,
\end{align*}
%
where $\mathcal F_t$ is the filtration that includes all the events $(x_{1:t+1},\xi^r_{1:t},\xi^c_{1:t})$ until the end of round $t$.
\end{assumption}
\begin{assumption}
\label{ass:bounded-reward-cost-param}
There is a known constant $S > 0$, such that $\| \theta_*\| \leq S$ and $\| \mu_*\| \leq S$.\footnote{The choice of the same upper-bounds for $\theta_*$ and $\mu_*$ is just for simplicity.}
\end{assumption}
\begin{assumption}
\label{ass:bounded-action}
The $\ell_2$-norm of all actions is bounded, i.e.,~$\max_{t\in[T]}\max_{x \in \mathcal{A}_t}\| x \|  \leq L$. 
\end{assumption}
\begin{assumption}
\label{ass:bounded-mean-reward-cost}
For all $t\in[T]$ and $x \in \mathcal{A}_t$, the mean rewards and costs are bounded, i.e.,~$\langle x, \theta_* \rangle \in [0,1]$ and $\langle x, \mu_* \rangle \in [0,1]$.
\end{assumption}
\begin{assumption}
\label{ass:safe-action}
There is a known safe action $x_0\in\mathcal A_t,\;\forall t\in[T]$ with known 
cost $c_0$, i.e.,~$\langle x_0, \mu_* \rangle = c_0 < \tau$. We will show how the assumption of knowing $c_0$ can be relaxed later in the paper.
\end{assumption}
{\bf Notation:} We conclude this section with introducing another set of notations that will be used in the rest of the paper. We define the normalized safe action as $e_0:=x_0/\|x_0\|$ and the span of the safe action as $\mathcal{V}_o := \mathrm{span}(x_0) =\{\eta x_0 : \eta\in\mathbb R\}$. We denote by $\mathcal{V}_o^{\perp}$, the orthogonal complement of $\mathcal{V}_o$, i.e.,~$\mathcal{V}_o^{\perp}=\{x\in\mathbb R^d : \langle x,y\rangle=0,\;\forall y\in\mathcal{V}_o\}$.\footnote{In the case of $x_0 = \mathbf{0}\in\mathbb R^d$, we define $\mathcal{V}_o$ as the empty subspace and $\mathcal{V}_o^\perp$ as the whole $\mathbb{R}^d$.} We define the projection of a vector $x\in\mathbb R^d$ into the sub-space $\mathcal{V}_o$, as $x^o:=\langle x,e_0\rangle e_0$, and into the sub-space $\mathcal{V}_o^{\perp}$, as $x^{o,\perp} := x - x^o$. We also define the projection of a policy $\pi$ into $\mathcal{V}_o$ and $\mathcal{V}_o^{\perp}$, as $x_{\pi}^o := \mathbb{E}_{x \sim \pi}[x^o]$ and $x_{\pi}^{o, \perp} := \mathbb{E}_{x \sim \pi}[x^{o, \perp}]$. 


\section{Optimistic-Pessimistic Linear Bandit Algorithm}
\label{sec:algo}

In this section, we propose an algorithm, called {\em optimistic-pessimistic linear bandit} (OPLB), whose pseudo-code is shown in Algorithm~\ref{alg:optimistic-pessimistic-LB}. Our OPLB algorithm balances a pessimistic assessment of the set of available policies, while acting optimistically within this set. Our principal innovation is the use of confidence intervals with asymmetric radii, proportional to $\alpha_r$ and $\alpha_c$, for the reward and cost signals. This will prove crucial in the regret analysis of the algorithm. 
 
{\centering
\begin{minipage}{.9\linewidth}
    \begin{algorithm}[H]
    \textbf{Input:} \begin{small}Horizon $T$, Confidence Parameter $\delta$, Regularization Parameter $\lambda$, Constants $\alpha_r, \alpha_c \geq 1$\end{small} \\
    \For{$t=1,\ldots,T$}{
    1. Compute RLS estimates $\;\widehat{\theta}_t\;$ and $\;\widehat{\mu}_t^{o,\perp}$ \hfill {\em (see Eqs.~\ref{eq:param-est} to~\ref{eq:proj-est})} \\
    2. Construct sets $\;\mathcal{C}_t^r(\alpha_r)\;$ and $\;\mathcal{C}_t^c(\alpha_c)$ \hfill {\em (see Eq.~\ref{eq:confidence-ellipsoids})} \\
    3. Observe $\;\mathcal{A}_t\;$ and construct the (estimated) safe policy set $\;\Pi_t$ \hfill {\em (see Eq.~\ref{eq:safe-policy-set})} \\
    4. Compute policy $\;(\pi_t,\widetilde{\theta}_t) = \argmax_{\pi \in \Pi_t,\;\theta \in \mathcal{C}_t^r(\alpha_r)} \mathbb{E}_{x \sim \pi}[ \langle x, \theta \rangle ]$\\
    5. Take action $\;x_t \sim \pi_t\;$ and observe reward and cost $\;(r_t,c_t)$
     }    
     \caption{Optimistic-Pessimistic Linear Bandit (OPLB)}
    \label{alg:optimistic-pessimistic-LB}
    \end{algorithm}
\end{minipage}
}

{\bf\em Line~1 of OPLB:} At each round $t\in[T]$, given the actions $\{x_s\}_{s=1}^{t-1}$, rewards $\{r_s\}_{s=1}^{t-1}$, and costs $\{c_s\}_{s=1}^{t-1}$ observed until the end of round $t-1$, OPLB first computes the $\ell_2$-regularized least-squares (RLS) estimates of $\theta_*$ and $\mu_*^{o,\perp}$ (projection of the cost parameter $\mu_*$ into the sub-space $\mathcal V_o^\perp$) as 

\vspace{-0.15in}
\begin{small}
\begin{equation}
\label{eq:param-est}
\widehat{\theta}_t = \Sigma_t^{-1} \sum_{s=1}^{t-1}r_sx_s, \qquad\quad \widehat{\mu}_t^{o,\perp} = (\Sigma_t^{o,\perp})^{-1}\sum_{s=1}^{t-1}c_s^{o,\perp}x_s^{o,\perp},    
\end{equation}
\end{small}
\vspace{-0.15in}

where $\lambda>0$ is the regularization parameter, and 

\vspace{-0.15in}
\begin{small}
\begin{align}
\label{eq:Sigmas}
&\Sigma_t = \lambda I + \sum_{s=1}^{t-1}x_sx_s^\top, \qquad\quad \Sigma_t^{o,\perp} = \lambda I_{\mathcal V_o^\perp} + \sum_{s=1}^{t-1}x_s^{o,\perp}(x_s^{o,\perp})^\top, \\ 
&c^{o,\perp}_t = c_t - \frac{\langle x_t,e_0\rangle}{\|x_0\|}c_0, \qquad\quad I_{\mathcal V_o^\perp} = I_{d\times d} - \frac{1}{\|x_0\|^2}x_0x_0^\top. 
\label{eq:proj-est}
\end{align}
\end{small}
\vspace{-0.15in}

In~\eqref{eq:Sigmas}, $\Sigma_t$ and $\Sigma_t^{o,\perp}$ are the Gram matrices of actions and projection of actions into the sub-space $\mathcal V_o^\perp$. Note that $\Sigma_t^{o, \perp}$ is a rank deficient matrix, but with abuse of notation, we use $(\Sigma_t^{o, \perp})^{-1}$ to denote its pseudo-inverse throughout the paper. In~\eqref{eq:proj-est}, $I_{\mathcal V_o^\perp}$ is the projection of the identity matrix, $I$, into $\mathcal V_o^\perp$, and $c^{o,\perp}_t$ is the noisy projection of the cost $c_t$ incurred by taking action $x_t$ into $\mathcal V_o^\perp$, i.e.,\footnote{In the derivation of~\eqref{eq:c-tilde-derivation}, we use the fact that $\langle x_t,\mu_*\rangle = \langle x_t^o+x_t^{o,\perp},\mu_*^o+\mu_*^{o,\perp}\rangle = \langle x_t^o,\mu_*^o\rangle + \langle x_t^{o,\perp},\mu_*^{o,\perp}\rangle$.}

\vspace{-0.15in}
\begin{small}
\begin{equation}
\label{eq:c-tilde-derivation}
c^{o,\perp}_t = \langle x_t^{o,\perp},\mu_*^{o,\perp}\rangle + \xi_t^c = \langle x_t,\mu_* \rangle - \langle x_t^o,\mu_*^o \rangle + \xi_t^c = c_t - \langle x_t^o,\mu_*^o \rangle = c_t - \frac{\langle x_t,e_0\rangle}{\|x_0\|}c_0.       
\end{equation}
\end{small}
\vspace{-0.125in}

{\bf\em Line~2:} Using the RLS estimates $\widehat{\theta}_t$ and $\widehat{\mu}_t^{o,\perp}$ in~\eqref{eq:param-est}, OPLB constructs the two {\em confidence sets} 

\vspace{-0.15in}
\begin{small}
\begin{equation}
\label{eq:confidence-ellipsoids}
\mathcal{C}_t^r(\alpha_r) = \big\{\theta \in \mathbb{R}^d: \|\theta -\widehat{\theta}_t\|_{\Sigma_t} \leq \alpha_r \beta_t(\delta, d)\big\}, \quad \mathcal{C}_t^c(\alpha_c) = \big\{ \mu \in \mathcal{V}_o^{\perp} : \| \mu - \widehat{\mu}^{o, \perp}_t\|_{\Sigma_t^{o,\perp}} \leq \alpha_c \beta_t(\delta, d-1)\big\},
\end{equation}
\end{small}
\vspace{-0.15in}

where $\alpha_r,\alpha_c\geq 1$ and $\beta_t(\delta,d)$ in the radii of these {\em confidence ellipsoids} is defined by the following theorem, originally proved in~\cite{abbasi2011improved}. 

\begin{theorem}\label{theorem::yasin_theorem}[Thm.~2 in~\citealt{abbasi2011improved}]
Let Assumptions~\ref{ass:noise-sub-gaussian} and~\ref{ass:bounded-reward-cost-param} hold, $\widehat{\theta}_t$, $\widehat{\mu}_t^{o,\perp}$, $\Sigma_t$, and $\Sigma_t^{o,\perp}$ defined by~\eqref{eq:param-est} and~\eqref{eq:Sigmas}, and $\mathcal{C}_t^r(\cdot)$ and $\mathcal{C}_t^c(\cdot)$ defined by~\eqref{eq:confidence-ellipsoids}. Then, for a fixed $\delta \in (0,1)$ and

\vspace{-0.1in}
\begin{small}
\begin{equation}
\label{eq:ellipsoid-radius}
\beta_t(\delta,d) = R \sqrt{d\log\Big(\frac{1 + (t-1)L^2/\lambda}{\delta}\Big)} + \sqrt{\lambda}\;S,
\end{equation}
\end{small}
\vspace{-0.125in}

with probability at least $1-\delta$ and for all $t \geq 1$, it holds that $\theta_*\in\mathcal{C}_t^r(1)$ and $\mu_*^{o,\perp}\in\mathcal{C}_t^c(1)$.
\end{theorem}

Since $\alpha_r,\alpha_c\geq 1$, for all rounds $t\in[T]$, the sets $\mathcal{C}_t^r(\alpha_r)$ and $\mathcal{C}_t^c(\alpha_c)$ also contain $\theta_*$, the reward parameter, and $\mu_*^{o,\perp}$, the projection of the cost parameter into $\mathcal V_o^\perp$, respectively, with high probability. 

Given these confidence sets, we define the {\em optimistic reward} and {\em pessimistic cost} of any policy $\pi$ in round $t$ as 

\vspace{-0.225in}
\begin{small}
\begin{equation}
\label{eq:max-reward-cost}
\widetilde{r}_{\pi,t} := \max_{\theta\in\mathcal{C}_t^r(\alpha_r)}   \mathbb{E}_{x \sim \pi}[\langle x,\theta\rangle], \qquad\qquad
\widetilde{c}_{\pi,t} := \frac{\langle x_{\pi}^o,e_0 \rangle c_0}{\| x_0 \|} + \max_{\mu\in\mathcal{C}_t^c(\alpha_c)} \mathbb{E}_{x\sim\pi}[\langle x,\mu \rangle].
\end{equation}
\end{small}
\vspace{-0.075in}

\begin{proposition}
\label{prop:optimistic-reward-pessimistic-cost}
We may write~\eqref{eq:max-reward-cost} in closed-form as \hfill {\em (proof in Appendix~\ref{subsec:proof-prop-optimistic-reward})}

\vspace{-0.15in}
\begin{small}
\begin{align}
\label{eq:optimistic-reward-closed-form}
\widetilde{r}_{\pi,t} &= \langle x_\pi,\widehat{\theta}_t \rangle + \alpha_r \beta_t(\delta,d) \| x_\pi\|_{\Sigma_t^{-1}}, \\ \widetilde{c}_{\pi,t} &= \frac{\langle x_\pi^o, e_0\rangle c_0 }{\|x_0\|} + \langle x_\pi^{o, \perp}, \widehat{\mu}^{o, \perp}_t \rangle + \alpha_c \beta_t(\delta,d-1) \| x_{\pi}^{o, \perp}\|_{(\Sigma_t^{o, \perp})^{-1}}.
\label{eq:pessimistic-cost-closed-form}
\end{align}
\end{small}
\vspace{-0.15in}

\end{proposition}


{\bf\em Line~3:} After observing the action set $\mathcal A_t$, OPLB constructs its (estimated) feasible (safe) policy set

\vspace{-0.15in}
\begin{small}
\begin{equation}
\label{eq:safe-policy-set}
\Pi_t = \{\pi \in \Delta_{\mathcal A_t} : \widetilde{c}_{\pi,t} \leq \tau\}, 
\end{equation}
\end{small}
\vspace{-0.175in}

where $\widetilde{c}_{\pi,t}$ is the pessimistic cost of policy $\pi$ in round $t$ defined by~\eqref{eq:pessimistic-cost-closed-form}. Note that $\Pi_t$ is not empty since $\pi_0$, the policy that plays the safe action $x_0$ with probability (w.p.)~$1$, is always in $\Pi_t$. This is because $x_{\pi_0}^o = x_0$, $x_{\pi_0}^{o,\perp} =0$, and $\frac{\langle x_{\pi_0}^o, e_0\rangle c_0}{\| x_0 \|}  = c_0$. In the following proposition, whose proof is reported in Appendix~\ref{subsec:proof-safe-set}, we prove that all policies in $\Pi_t$ are feasible with high probability.

\begin{proposition}
\label{prop:safe-set}
With probability at least $1-\delta$, for all rounds $t\in[T]$, all policies in $\Pi_t$ are feasible. 
\end{proposition}

{\bf\em Line~4:} The agent computes its policy, $\pi_t$, as the one that is safe (belongs to $\Pi_t$) and attains the maximum optimistic reward. We refer to $\widetilde{\theta}_t$ as the {\em optimistic reward parameter}. Thus, we write the optimistic reward of policy $\pi_t$ as $\widetilde{r}_{\pi_t,t} = \langle x_{\pi_t},\widetilde{\theta}_t \rangle$.

{\bf\em Line~5:} Finally, the agent selects an action $x_t\sim\pi_t$ and observes the reward-cost pair $(r_t,c_t)$.

{\bf Computational Complexity of OPLB.} $\;$ As shown in Line~4 of Algorithm~\ref{alg:optimistic-pessimistic-LB} and in Proposition~\ref{prop:optimistic-reward-pessimistic-cost}, in each round $t$, OPLB solves the following optimization problem:

\vspace{-0.15in}
\begin{small}
\begin{align}
\label{equation::opt_problem_linear}
\max_{\pi \in \Delta_{\mathcal{A}_t}}&  \langle x_\pi, \widehat{\theta}_t \rangle + \alpha_r \beta_t(\delta,d) \| x_\pi\|_{\Sigma_t^{-1}} \\ 
\text{s.t. }&   \frac{\langle x_\pi^o, e_0\rangle c_0 }{\|x_0\|} + \langle x_\pi^{o, \perp}, \widehat{\mu}^{o, \perp}_t \rangle + \alpha_c \beta_t(\delta,d-1) \| x_{\pi}^{o, \perp}\|_{(\Sigma_t^{o, \perp})^{-1}} \leq \tau. \nonumber
\end{align}
\end{small}
\vspace{-0.15in}

However, solving~\eqref{equation::opt_problem_linear} can be challenging. The bottleneck is computing the safe policy set $\Pi_t$, which is the intersection between $\Delta_{\mathcal{A}_t}$ and the ellipsoidal constraint.

\begin{remark}
The main challenge in obtaining a regret bound for OPLB is to ensure that optimism holds in each round $t$, i.e.,~the solution $(\pi_t,\widetilde\theta_t)$ of~\eqref{equation::opt_problem_linear} satisfy $\widetilde{r}_{\pi_t,t} = \langle x_{\pi_t},\widetilde\theta_t \rangle \geq r_{\pi^*_t}$. This is not obvious, since the (estimated) safe policy set $\Pi_t$ may not contain the optimal policy $\pi_t^*$. Our main algorithmic innovation is the use of asymmetric confidence intervals $\mathcal{C}_t^r(\alpha_r)$ and $\mathcal{C}_t^c(\alpha_c)$ for $\theta_*$ and $\mu^{o, \perp}_*$, which allows us to guarantee optimism, by appropriately selecting the ratio $\gamma = \alpha_r/\alpha_c$. Of course, this comes at the cost of scaling the regret by a factor $\gamma$. As it will be shown in our analysis in Section~\ref{section::lin_opt_pess_analysis}, $\gamma$ depends on the inverse gap $1/(\tau-c_0)$, which indicates when $\tau - c_0$ is small (the cost of the safe arm is close to the constraint threshold), the agent will have a difficult time to identify a safe arm and to compete against the optimal feasible policy $\pi_t^*$. We will formalize this in Lemma \ref{lemma:linear_bandits_optimism}. 
\end{remark}

\begin{remark}
If the cost of the safe arm $c_0$ is unknown, we start by taking the safe action $x_0$ for $T_0$ rounds to produce a conservative estimate $\hat\delta_c$ of $\tau - c_0$ that satisfies $\hat\delta_c \geq \frac{\tau - c_0}{2}$. We warm start our estimators for $\theta_*$ and $\mu_*$ using the data collected by playing $x_0$. However, instead of estimating $\mu_*^{o, \perp}$, we build an estimator for $\mu_*$ over all its directions, including $e_0$, similar to what OPLB does for $\theta_*$. We then set $\frac{\alpha_r}{\alpha_c} = 1/\hat\delta_c$ and run Algorithm~\ref{alg:optimistic-pessimistic-LB} for rounds $t > T_0$ (see Appendix~\ref{section::safe_policy_value} for more details). 
\end{remark}


\section{Regret Analysis}
\label{section::lin_opt_pess_analysis}

In this section, we prove the following regret bound for OPLB (Algorithm~\ref{alg:optimistic-pessimistic-LB}). 

\begin{theorem}[Regret of OPLB]
\label{theorem::main_theorem_linear}
Let $\alpha_c = 1$ and $\alpha_r = \frac{ 2+\tau-c_0}{\tau-c_0}$. Then, with probability at least $1-2\delta$, the regret of OPLB satisfies 

\vspace{-0.15in}
\begin{small}
\begin{equation}
\label{eq:regret-OPLB}
\mathcal{R}_{\Pi}(T) \leq \frac{2L(\alpha_r+1)\beta_T(\delta, d)}{\sqrt{\lambda}}\sqrt{2T\log(1/\delta)} + (\alpha_r+1)\beta_T(\delta, d)\sqrt{2Td\log(1+\frac{TL^2}{\lambda})}.
\end{equation}
\end{small}
\vspace{-0.15in}
\end{theorem}

We start the proof of Theorem~\ref{theorem::main_theorem_linear}, by defining the following event that holds w.p.~at least $1-\delta$: 

\vspace{-0.125in}
\begin{small}
\begin{equation}
\label{eq:high-prob-event}
\mathcal{E} := \big\{\|\widehat{\theta}_t - \theta_*\|_{\Sigma_t} \leq \beta_t(\delta, d) \; \wedge \; \|\widehat{\mu}_t^{o, \perp} - \mu_*^{o, \perp}\|_{\Sigma_t^{o,\perp}} \leq \beta_t(\delta, d-1), \; \forall t\in[T]\big\}.
\end{equation}
\end{small}
\vspace{-0.15in}

The regret $\mathcal{R}_{\Pi}(T)$ in~\eqref{eq:regret} can be decomposed as ($\widetilde{r}_{\pi_t,t}$ is the optimistic reward defined by Eq.~\ref{eq:max-reward-cost})

\vspace{-0.15in}
\begin{small}
\begin{equation}
\label{eq:regret-decomp}
\mathcal{R}_{\Pi}(T) = \underbrace{\sum_{t=1}^T r_{\pi_t^*} - \widetilde{r}_{\pi_t,t}}_{(\mathrm{I})} \; + \; \underbrace{\sum_{t=1}^T \widetilde{r}_{\pi_t,t} - r_{\pi_t}}_{(\mathrm{II})}.
\end{equation}
\end{small}
\vspace{-0.125in}

We first bound the term $(\mathrm{II})$ in~\eqref{eq:regret-decomp}. To bound $(\mathrm{II})$, we further decompose it as

\vspace{-0.15in}
\begin{small}
\begin{equation}
\label{eq:II-decomp}
(\mathrm{II}) = \underbrace{\sum_{t=1}^T  \langle x_{\pi_t}, \widetilde{\theta}_t \rangle - \langle x_t, \widetilde{\theta}_t \rangle}_{(\mathrm{III})} + \underbrace{\sum_{t=1}^T \langle x_t, \widetilde{\theta}_t \rangle - \langle x_t, \theta_* \rangle}_{(\mathrm{IV})} + \underbrace{\sum_{t=1}^T \langle x_t, \theta_* \rangle - \langle x_{\pi_t}, \theta_* \rangle }_{(\mathrm{V})}.
\end{equation}
\end{small}
\vspace{-0.125in}

In the following lemmas, we first bound the sum of $(\mathrm{III})$ and $(\mathrm{V})$ terms, and then bound $(\mathrm{IV})$.

\begin{lemma}
\label{lemma:bounding-3+5}
On the event $\mathcal{E}$ defined by~\eqref{eq:high-prob-event}, for any $\gamma \in (0, 1)$, w.p.~at least $1-\gamma$, we have

\vspace{-0.125in}
\begin{small}
\begin{equation*}
(\mathrm{III}) + (\mathrm{V}) \leq \frac{2L(\alpha_r + 1)\beta_T(\delta, d)}{\sqrt{\lambda}}\cdot\sqrt{2 T \log(1/\gamma)}\;.
\end{equation*}
\end{small}
\end{lemma}

\vspace{-0.175in}
\begin{proof}
We write $(\mathrm{III}) + (\mathrm{V}) = \sum_{t=1}^T \langle x_{\pi_t} - x_t, \widetilde{\theta}_t - \theta_* \rangle$. By Cauchy-Schwartz, we have $|\langle x_{\pi_t} - x_t, \widetilde{\theta}_t - \theta_* \rangle | \leq \| x_{\pi_t} - x_t \|_{\Sigma^{-1}_t} \| \widetilde{\theta}_t - \theta_*\|_{\Sigma_t}$. Since $\widetilde{\theta}_t\in \mathcal{C}_t^r(\alpha_r)$, on event $\mathcal E$, we have $\|\widetilde{\theta}_t - \theta_*\|_{\Sigma_t} \leq (\alpha_r+1) \beta_t(\delta, d)$. Also from the definition of $\Sigma_t$, we have $\Sigma_t \succeq \lambda I$, and thus, $\| x_{\pi_t} - x_t \|_{\Sigma^{-1}_t} \leq \| x_{\pi_t} - x_t\|/\sqrt{\lambda} \leq 2L/\sqrt{\lambda}$. Therefore, $Y_t =\sum_{s=1}^t \langle x_{\pi_s} - x_s, \widetilde{\theta}_s - \theta_* \rangle$ is a martingale sequence with $|Y_t - Y_{t-1}| \leq 2L(\alpha_r + 1)\beta_t(\delta, d)/\sqrt{\lambda}$, for $t\in [T]$. By the Azuma–Hoeffding inequality and since $\beta_t$ is an increasing function of $t$, i.e.,~$\beta_t(\delta, d) \leq \beta_T(\delta, d),\;\forall t\in[T]$, w.p.~at least  $1-\gamma$, we have $\mathbb{P}\big(Y_T \geq 2L(\alpha_r + 1) \beta_T(\delta, d)\sqrt{2T \log(1/\gamma)/\lambda}\big) \leq \gamma$, which concludes the proof. 
\end{proof}

\begin{lemma}
\label{lemma:bounding-4}
On event $\mathcal{E}$, we have $(\mathrm{IV}) \leq (\alpha_r+1) \beta_T(\delta,d)\sqrt{2Td\log\big(1+\frac{TL^2}{\lambda}\big)}$. \end{lemma}

We report the proof of Lemma~\ref{lemma:bounding-4} in Appendix~\ref{subsec:proof-Lemma-bounding-4}. After bounding all the terms in $(\mathrm{II})$, we now process the term $(\mathrm{I})$ in~\eqref{eq:regret-decomp}. Before stating the main result for this term  in Lemma~\ref{lemma:linear_bandits_optimism}, we need to prove the following lemma (proof in Appendix~\ref{subsec:proof-Lemma-inverse-norm-domination}). 
\begin{lemma}
\label{lemma:inverse_norm_domination}

For any policy $\pi$, the following inequality holds: 

\vspace{-0.125in}
\begin{small}
\begin{equation}
\label{equation::comparison_conf_bounds}
\| x_\pi^{o, \perp}\|_{(\Sigma_t^{o,\perp})^{-1}} \leq \| x_\pi \|_{\Sigma_t^{-1}}.
\end{equation}
\end{small}
\end{lemma}

In the following lemma, we prove that by appropriately setting the parameters $\alpha_r$ and $\alpha_c$, we can guarantee that at each round $t\in[T]$, OPLB selects an optimistic policy, i.e.,~a policy $\pi_t$, whose optimistic reward, $\widetilde{r}_{\pi_t,t}$, is larger than the reward of the optimal policy $r_{\pi^*_t}$, given the event $\mathcal E$. This means that with our choice of parameters $\alpha_r$ and $\alpha_c$, the term $(\mathrm{I})$ in~\eqref{eq:regret-decomp} is always non-positive.

\begin{lemma}
\label{lemma:linear_bandits_optimism}
On the event $\mathcal{E}$, if we set $\alpha_r$ and $\alpha_c$, such that $\alpha_r,\alpha_c\geq 1$ and $1+\alpha_c \leq (\tau-c_0) (\alpha_r-1)$, then for any $t\in[T]$, we have $\widetilde{r}_{\pi_t,t} \geq r_{\pi_t^*}$.
\end{lemma}

Here we provide a proof sketch for Lemma~\ref{lemma:linear_bandits_optimism}. The detailed proof is reported in Appendix~\ref{subsec:proof-Lemma-term-I}. 

\begin{proof}[Proof Sketch]
We divide the proof into two cases, depending on whether in each round $t$, the optimal policy $\pi^*_t$ belongs to the (estimated) set of feasible policies $\Pi_t$, or not. 

{\bf Case~yes
1.} $\;$ If $\pi_t^*\in\Pi_t$, then its optimistic reward is less than that of the policy $\pi_t$ selected at round $t$ (by the definition of $\pi_t$ on Line~4 of Algorithm~\ref{alg:optimistic-pessimistic-LB}), i.e.,~$\widetilde{r}_{\pi_t^*,t}\leq\widetilde{r}_{\pi_t,t}$. This together with the fact that the optimistic reward of any policy $\pi$ is larger than its expected reward, i.e.,~$\widetilde{r}_{\pi,t}\geq r_\pi$, gives us the desired result that $\widetilde{r}_{\pi_t,t}\geq r_{\pi^*_t}$. 

{\bf Case~2.} $\;$ If $\pi_t^*\not\in\Pi_t$, then we define a mixture policy $\widetilde{\pi}_t=\eta_t\pi_t^*+(1-\eta_t)\pi_0$, where $\pi_0$ is the policy that always selects the safe action $x_0$ and $\eta_t\in[0,1]$ is the maximum value of $\eta$ for which the mixture policy belongs to the set of feasible actions, i.e.,~$\widetilde{\pi}_t\in\Pi_t$. Conceptually, we can think of $\eta_t$ as a measure for safety of the optimal policy $\pi_t^*$. Mathematically, $\eta_t$ is the value at which the pessimistic cost of the mixture policy equals to the constraint threshold, i.e.,~$\widetilde{c}_{\widetilde{\pi}_t,t} =\tau$. In the rest of the proof, we first write $\widetilde{c}_{\widetilde{\pi}_t,t}$ in terms of the pessimistic cost of the optimal policy as $\widetilde{c}_{\widetilde{\pi}_t,t}=(1-\eta_t)c_0+\eta_t\widetilde{c}_{\pi^*_t,t}$ ($c_0$ is the expected cost of the safe action $x_0$), and find a lower-bound for $\eta_t$ (see Eq.~\ref{eq:tempB3} in Appendix~\ref{subsec:proof-Lemma-term-I}). We then use the fact that since $\widetilde{\pi}_t\in\Pi_t$, its optimistic reward is less than that of $\pi_t$, i.e.,~$\widetilde{r}_{\pi_t,t}\geq \widetilde{r}_{\widetilde{\pi}_t,t}$, and obtain a lower-bound for  $\widetilde{r}_{\widetilde{\pi}_t,t}$ as a function of $r_{\pi_t^*}$ (see Eq.~\ref{eq:tempB4} in Appendix~\ref{subsec:proof-Lemma-term-I}). Finally, we conclude the proof by using this lower-bound and finding the relationship between the parameters $\alpha_r$ and $\alpha_c$ for which the desired result $\widetilde{r}_{\pi_t,t}\geq r_{\pi_t^*}$ is obtained, i.e.,~$1+\alpha_c\leq (\tau-c_0)(\alpha_r-1)$.
\end{proof}

\begin{proof}[Proof of Theorem~\ref{theorem::main_theorem_linear}]
The proof follows from the fact that the term $(\mathrm{I})$ is negative (Lemma~\ref{lemma:linear_bandits_optimism}), and by combining the upper-bounds on the term $(\mathrm{II})$ from Lemmas~\ref{lemma:bounding-3+5} and~\ref{lemma:bounding-4}, and setting $\gamma = \delta$.
\end{proof}





%% file: multi_armed_bandits.tex
\section{Constrained Multi-Armed Bandits}
\label{sec:constrained-MAB}

In this section, we specialize our results for contextual linear bandits to multi-armed bandits (MAB) and show that the structure of the MAB problem allows a computationally efficient implementation of the algorithm and an improvement in the regret bound.  

In the MAB setting, the action set consists of $K$ arms $\mathcal{A}=\{1,\ldots,K\}$. Each arm $a\in [K]$ has a reward and a cost distribution with means $\bar{r}_a,\bar{c}_a\in [0,1]$. In each round $t\in[T]$, the agent constructs a policy $\pi_t$ over $\mathcal A$, pulls an arm $a_t\sim\pi_t$, and observes a reward-cost pair $(r_{a_t},c_{a_t})$ sampled i.i.d.~from the reward and cost distributions of arm $a_t$. Similar to the constrained contextual linear case, the goal of the agent is to produce a sequence of policies $\{\pi_t\}_{t=1}^T$ with maximum expected cumulative reward over $T$ rounds, i.e.,~$\sum_{t=1}^T\mathbb E_{a_t\sim\pi_t}[\bar{r}_{a_t}]$, while satisfying the {\bf\em linear constraint} $\mathbb E_{a_t\sim\pi_t}[\bar{c}_{a_t}] \leq \tau,\;\forall t\in[T]$. Moreover, arm $1$ is assumed to be the known safe arm, i.e.,~$\bar{c}_1 \leq \tau$. 

{\bf Optimistic Pessimistic Bandit (OPB) Algorithm.} $\;$ Let $\{T_a(t)\}_{a=1}^K$ and $\{\widehat{r}_a(t),\widehat{c}_a(t)\}_{a=1}^K$ be the total number of times that arm $a$ has been pulled and the estimated mean reward and cost of arm $a$ up until round $t$. In each round $t\in[T]$, OPB relies on the high-probability upper-bounds on the mean reward and cost of the arms, i.e.,~$\{u_a^r(t),u_a^c(t)\}_{a=1}^K$, where $u_a^r(t) = \widehat{r}_a(t) + \alpha_r\beta_a(t)$, $u_a^c(t) = \widehat{c}_a(t) + \alpha_c\beta_a(t)$, $\beta_a(t) = \sqrt{2\log(1/\delta')/T_a(t)}$, and constants $\alpha_r,\alpha_c\geq 1$. In order to produce a feasible policy, OPB solves the following linear program (LP) in each round $t\in[T]$: 

\vspace{-0.125in}
\begin{small}
\begin{equation}
\label{eq::noisy_LP}
\max_{\pi\in\Delta_K} \; \sum_{a\in\mathcal{A}}\pi_a \; u^r_a(t), \qquad\quad \text{s.t.} \; \sum_{a\in\mathcal{A}}\pi_a \; u_a^c(t) \leq \tau.
\end{equation}
\end{small}
\vspace{-0.125in}

As shown in~\eqref{eq::noisy_LP}, OPB selects its policy by being optimistic about reward (using an upper-bound for $r$) and pessimistic about cost (using an upper-bound for $c$). We report the details of OPB and its pseudo-code (Algorithm~\ref{alg::optimism_pessimism}) in Appendix~\ref{section::appendix_MAB_optimism_pessimism_algorithm}.


{\bf Computational Complexity of OPB.} $\;$ Unlike OPLB, whose optimization problem might be complex, OPB can be implemented extremely efficiently. Lemma~\ref{lemma::LP_support}, whose proof we report in Appendix~\ref{section::LP_structure_appendix}, show that~\eqref{eq::noisy_LP} always has a solution (policy) with support of at most $2$. This property allows us to solve~\eqref{eq::noisy_LP} in closed form, without a LP solver, and implement OPB quite efficiently. 

\begin{lemma}
\label{lemma::LP_support} 
There exists a policy that solves~\eqref{eq::noisy_LP} and has at most $2$ non-zero entries.
\end{lemma}


{\bf Regret Analysis of OPB.} $\;$ We prove the following regret-bound for OPB in Appendix~\ref{section:regret_analysis_appendix}. 

\begin{theorem}[Regret of OPB]
\label{theorem::contrained_MAB}
Let $\delta = 4KT\delta'$, $\alpha_c=1$, and $\alpha_r = 1+2/(\tau-\bar{c}_1)$. Then, with probability at least $1-\delta$, the regret of OPB satisfies

\vspace{-0.15in}
\begin{small}
\begin{equation*}
\mathcal{R}_\Pi(T) \leq \big(1 + \frac{2}{\tau-\bar{c}_1}\big) \times \big(2\sqrt{2KT\log(4KT/\delta)} + 4\sqrt{T\log(2/\delta)\log(4KT/\delta)}\big).
\end{equation*}
\end{small}
\vspace{-0.15in}
\end{theorem}


The main component in the proof of Theorem~\ref{theorem::contrained_MAB} is the following lemma, whose proof is reported in Appendix~\ref{section:regret_analysis_appendix}. This lemma is the analogous to Lemma~\ref{lemma:linear_bandits_optimism} in the contextual linear bandit case. 

\begin{lemma}
\label{lemma::optimism}
If we set the parameters $\alpha_r$ and $\alpha_c$, such that $\alpha_r,\alpha_c\geq 1$ and $\alpha_c \leq (\tau-\bar{c}_1)(\alpha_r-1)$, then with high probability, for any $t\in[T]$, we have $\mathbb{E}_{a \sim \pi_t}\left[u_a^r(t)\right] \geq \mathbb{E}_{a \sim \pi^*}\left[\bar{r}_a \right]$.
\end{lemma}

Our contextual linear bandit formulation is general enough to include MAB. The regret analysis of OPLB (Theorem~\ref{theorem::main_theorem_linear}) yields a regret bound of order $\widetilde{\mathcal O}(\frac{K\sqrt{T}}{\tau-\bar{c}_1})$ for MAB. However, our OPB regret bound in Theorem~\ref{theorem::contrained_MAB} is of order $\widetilde{\mathcal O}(\frac{\sqrt{KT}}{\tau-\bar{c}_1})$, which shows a $\sqrt{K}$ improvement over simply casting MAB as an instance of contextual linear bandit and using the regret bound of OPLB. 


{\bf Extension to $m$ Constraints.} $\;$ In this case, the agent receives $m$ cost signals after pulling each arm. The cost vector of the safe arm $\boldsymbol{c}_1$ satisfies $\boldsymbol{c}_1(i)<\tau_i,\forall i\in[m]$, where $\{\tau_i\}_{i=1}^m$ are the constraint thresholds. Similar to single-constraint OPB, multi-constraint OPB is also computationally efficient. The main reason is that the LP of $m$-constraint OPB has a solution with at most $m+1$ non-zero entries. We obtain a regret bound of $\widetilde{\mathcal{O}}(\frac{\sqrt{KT }}{ \min_i \tau_i - \boldsymbol{c}_1(i)})$ for $m$-constraint OPB in Appendix~\ref{section::multiple_constraints_appendix}.



{\bf Lower-bound.} $\;$ We also prove a mini-max lower-bound for this constrained MAB problem that shows no algorithm can attain a regret better than $\mathcal{O}(\max(\sqrt{KT}, \frac{1}{(\tau - \bar{c}_1)^2} ) )$. The formal statement of the lower-bound and the proof are reported in Appendix~\ref{section::lower_bound_MAB_appendix}.

%% file: experiments.tex
\section{Experiments}
\label{sec:experiments}
\vspace{-0.1in}

We run a set of experiments to show the behavior of OPB and validate our theoretical results. We consider a $K=4$-armed bandits in which the reward and cost distributions of the arms are Bernoulli with means $\bar r=(.1,.2,.4,.7)$ and $\bar c =(0,.4,.5,.2)$. So, the cost of the safe arm is $\bar c_1=0$. In Figures~\ref{fig:constrained_bandits1} to~\ref{fig:constrained_bandits3}, we gradually reduce the constraint threshold $\tau$, and as a result the complexity of the problem $\tau - \bar{c}_1$, and show the regret {\em (left)} and the cost {\em (middle)} and reward {\em (right)} evolution of OPB. All the results are averaged over $10$ runs and the shade is the $\pm .5$ standard deviation around the regret. 

Our results show that the regret of OPB grows as we reduce $\tau$ {\em (left)}. They also indicate that the algorithm is successful in satisfying the constraint {\em (middle)} and reaching the optimal reward/performance {\em (right)}. In Figure~\ref{fig:constrained_bandits3}, the reason that the cost evolution of OPB is the same as that of the optimal policy {\em (middle)} is that in this case, the cost of the best arm (arm $4$) is equal to the constraint threshold $\tau=.2$. 


\begin{wrapfigure}{r}{0.625\textwidth}
        \vspace{-0.3in}    
        \begin{minipage}{0.625\textwidth}
        \centering\subfigure{\includegraphics[width=0.325\linewidth]{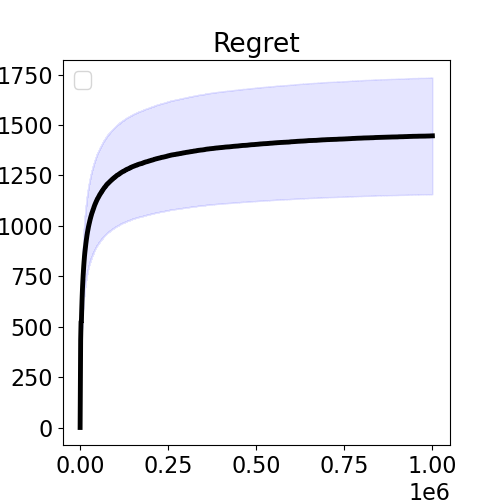}} 
        \centering\subfigure{\includegraphics[width=0.325\linewidth]{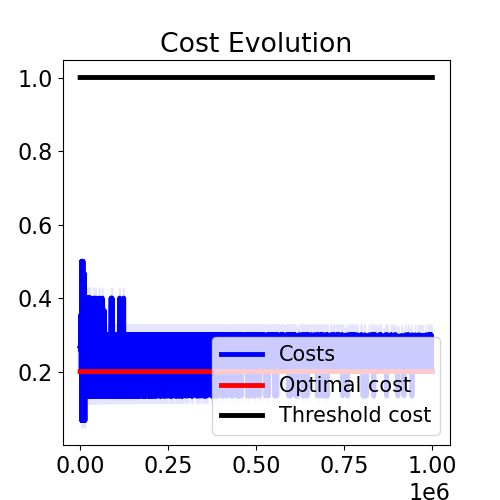}} 
        \centering\subfigure{\includegraphics[width=0.325\linewidth]{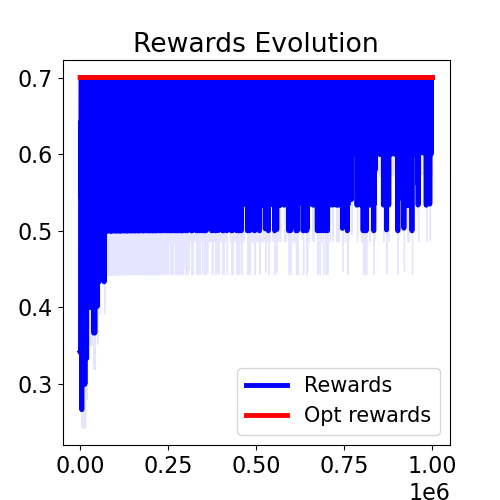}} 
        \end{minipage}
        \vspace{-4mm}
        \caption{{Constraint Threshold $\;\tau = 1$.}}
        \vspace{-2mm}
        \label{fig:constrained_bandits1}
        \begin{minipage}{0.625\textwidth}
        \centering\subfigure{\includegraphics[width=0.325\linewidth]{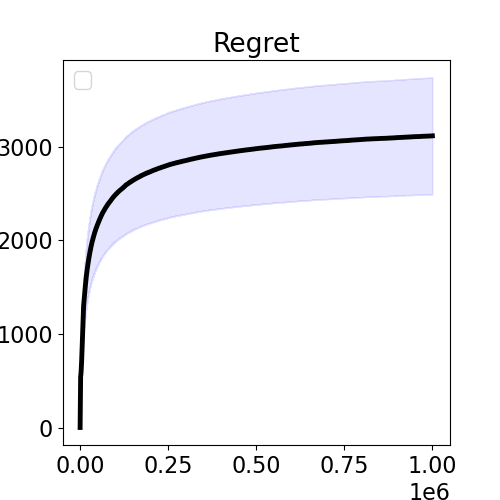}} 
        \centering\subfigure{\includegraphics[width=0.325\linewidth]{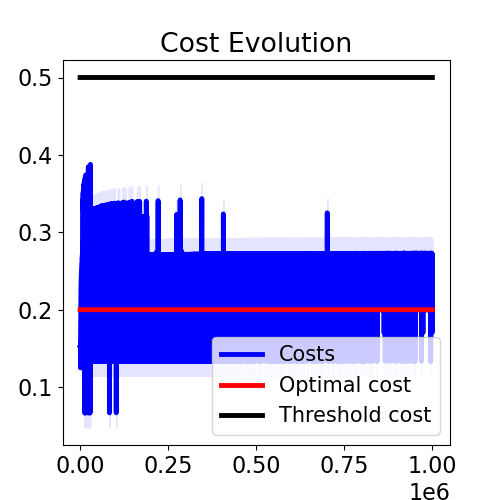}} 
        \centering\subfigure{\includegraphics[width=0.325\linewidth]{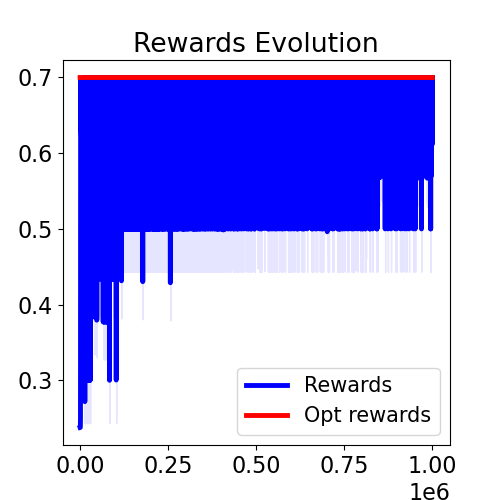}} 
        \end{minipage}
        \vspace{-4mm}
        \caption{{Constraint Threshold $\;\tau = 0.5$.}}
        \vspace{-2mm}
        \label{fig:constrained_bandits2}
        \begin{minipage}{0.625\textwidth}
        \centering\subfigure{\includegraphics[width=0.325\linewidth]{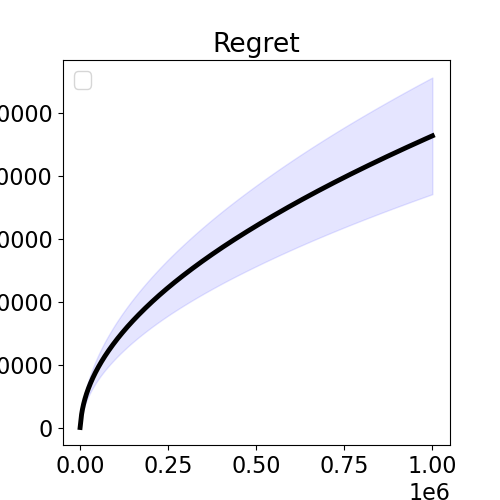}} 
        \centering\subfigure{\includegraphics[width=0.325\linewidth]{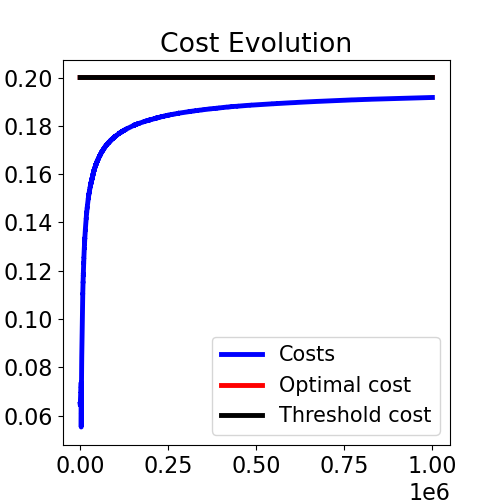}} 
        \centering\subfigure{\includegraphics[width=0.325\linewidth]{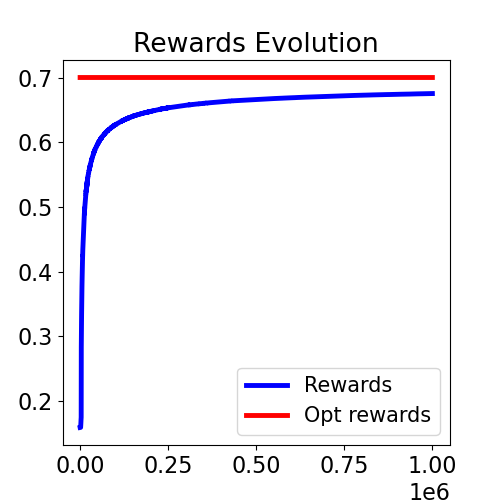}} 
        \end{minipage}
        \vspace{-4mm}
        \caption{{Constraint Threshold $\;\tau = 0.2$.}}
        \label{fig:constrained_bandits3}
\begin{center}
\begin{minipage}{0.625\textwidth}
\begin{center}   \small{ \textbf{OPB.} Bernoulli arms. $\bar{r} = (.1, .2, .4, .7)$, $\bar{c}=(0, .4, .5, .2)$, $\bar{c}_1 = 0$.} \end{center}
    \end{minipage}
    \end{center}
\vspace{-0.225in}    
\end{wrapfigure}


%% file: previous_work.tex
\section{Related Work}
\label{sec:related-work}
\vspace{-0.1in}

As described in Section~\ref{sec:intro}, our setting is the closest to the one studied by~\cite{amani2019linear} and~\cite{Moradipari19SL}. They study a slightly different setting, in which the mean cost of the action that the agent takes should satisfy the constraint, i.e.,~$\langle x_t,\mu_*\rangle \leq \tau$, not the mean cost of the policy it computes, i.e.,~$\langle x_{\pi_t},\mu_*\rangle \leq \tau$, as in our case. Clearly, the setting studied in our paper is more relaxed, and thus, is expected to obtain more rewards.~\cite{Moradipari19SL} propose a TS algorithm for their setting and prove an $\widetilde{O}(d^{3/2}\sqrt{T}/\tau)$ regret bound for it. They restrict themselves to linear bandits, i.e.,~$\mathcal A_t=\mathcal A,\forall t\in[T]$, and the safe action being the origin, i.e.,~$x_0=\mathbf{0}$ and $c_0=0$. This is why $c_0$ does not appear in their bounds. They consider their action set to be any convex compact subset of $\mathbb R^d$ that contains the origin. Although later in their proofs, to guarantee that their algorithm does not violate the constraint in the first round, they require the action set to also contain the ball with radius $\tau/S$ around the origin. Therefore, our action set is more general than theirs. Moreover, unlike us, their action set does not allows their results to be immediately applicable to MAB. Our regret bound also has a better dependence on $d$ and $\log T$ than theirs, similar to the best regret results for UCB vs.~TS. However, their algorithm is TS, and thus, is less complex than ours. Although it can be still intractable, even when $\mathcal A$ is convex. They needed to do several approximations in order to make their algorithm tractable in their experiments. 


In~\cite{amani2019linear}, reward and cost have the same unknown parameter $\theta_*$, and the cost is defined as $c_t=x_t^\top B\theta_*\leq\tau$, where $B$ is a known matrix. They derive and analyze an explore-exploit algorithm for this setting. Although our rate is better than theirs, i.e.,~$\widetilde{O}(T^{2/3})$, our algorithm cannot immediately give a $\widetilde{O}(\sqrt{T})$ regret for their setting, unless in special cases.

\section{Conclusions}
\label{sec:conclu}
\vspace{-0.1in}

We derived a UCB-style algorithm for a new constrained contextual linear bandit setting, in which the goal is to produce a sequence of policies with maximum expected cumulative reward, while each policy has an expected cost below a certain threshold $\tau$. We proved a $T$-round regret bound of $\widetilde{\mathcal{O}}(\frac{d\sqrt{T}}{\tau-c_0})$ for our algorithm, which shows that the difficulty of the problem depends on the difference between the constraint threshold and the cost of a known feasible action $c_0$. We further specialized our results to MAB and proposed and analyzed a computationally efficient algorithm for this setting. We also proved a lower-bound for our constrained bandit problem and provided simulations to validate our theoretical results. A future direction is to use the optimism-pessimism idea behind our algorithm in other constrained bandit settings, including deriving a UCB-style algorithm for the setting studied in~\cite{amani2019linear} and~\cite{Moradipari19SL}. 


%% file: appendix_linear.tex
\section{Proofs of Section~\ref{sec:algo}}
\label{sec:proofs-algo-section}


\subsection{Proof of Proposition~\ref{prop:optimistic-reward-pessimistic-cost}}
\label{subsec:proof-prop-optimistic-reward}

\begin{proof}
We only prove the statement for the optimistic reward, $\widetilde{r}_{\pi,t}$. The proof for the pessimistic cost, $\widetilde{c}_{\pi,t}$, is analogous. From the definition of the confidence set $\mathcal{C}_t^r(\alpha_r)$ in~\eqref{eq:confidence-ellipsoids}, any vector $\theta \in \mathcal{C}_t^r(\alpha_r)$ can be written as $\widehat{\theta}_t + v$, where $v$ satisfying $\| v\|_{\Sigma_t} \leq \alpha_r \beta_t(\delta, d)$. Thus, we may write 

\begin{align*}
\widetilde{r}_{\pi,t} &=  \max_{\theta \in \mathcal{C}_t^r(\alpha_r)} \mathbb{E}_{x \sim \pi}[ \langle x, \theta \rangle ] = \max_{\theta \in \mathcal{C}_t^r(\alpha_r)}  \langle x_\pi, \theta \rangle = \langle x_\pi, \widehat{\theta}_t \rangle + \max_{v:\|v\|_{\Sigma_t} \leq \alpha_r \beta_t(\delta, d)} \langle x_\pi, v \rangle \\
&\stackrel{\text{(a)}}{\leq} \langle x_\pi,\widehat{\theta}_t \rangle + \alpha_r \beta_t(\delta, d) \|  x_\pi \|_{\Sigma_t^{-1}}.
\end{align*}

{\bf (a)} By Cauchy-Schwartz, for all $v$, we have $\langle x_{\pi}, v \rangle \leq \| x_{\pi}\|_{\Sigma_t^{-1}} \| v\|_{\Sigma_t}$. The result follows from the condition on $v$ in the maximum, i.e.,~$\| v \|_{\Sigma_t} \leq \alpha_r \beta_t(\delta,d)$.

Let us define $v^* := \frac{ \alpha_r \beta_t(\delta, d) \Sigma^{-1}_t x_\pi}{\| x_\pi \|_{\Sigma^{-1}_t}}$. This value of $v^*$ is feasible because 
\begin{equation*}
\|v^*\|_{\Sigma_t} = \frac{\alpha_r \beta_t(\delta,d)}{\| x_\pi \|_{\Sigma^{-1}_t}} \sqrt{x_\pi^\top \Sigma_t^{-1} \Sigma_t \Sigma_t^{-1} x_\pi} = \frac{\alpha_r\beta_t(\delta,d)}{\|x_\pi\|_{\Sigma^{-1}_t}} \sqrt{x_\pi^\top\Sigma_t^{-1}x_\pi} = \alpha_r\beta_t(\delta, d).
\end{equation*}
We now show that $v^*$ also achieves the upper-bound in the above inequality resulted from Cauchy-Schwartz 
\begin{equation*}
\langle x_\pi, v^* \rangle =  \frac{\alpha_r \beta_t(\delta, d)  x_\pi^\top \Sigma_t^{-1} x_\pi}{\| x_\pi \|_{\Sigma_t^{-1}}}   = \alpha_r \beta_t(\delta, d) \|  x_\pi \|_{\Sigma_t^{-1}}.
\end{equation*}
Thus, $v^*$ is the maximizer and we can write 
\begin{align*}
\widetilde{r}_{\pi,t} = \langle x_\pi,\widehat{\theta}_t \rangle + \langle x_\pi,v^* \rangle = \langle x_\pi,\widehat{\theta}_t \rangle + \alpha_r \beta_t(\delta, d) \|  x_\pi \|_{\Sigma_t^{-1}},
\end{align*}
which concludes the proof. 
\end{proof}


\subsection{Proof of Proposition~\ref{prop:safe-set}}
\label{subsec:proof-safe-set}


\begin{proof}
Recall that $\tilde{c}_{\pi, t} = \frac{ \langle x_{\pi}^o, e_0\rangle c_0}{ \| x_0\| }  + \langle x_{\pi}^{o, \perp}, \widehat{t}_{\pi}^{o, \perp}\rangle  + \alpha_c \beta_t(\delta, d-1)\| x_{\pi}^{o, \perp} \|_{ (\Sigma^{o, \perp}_t)^{-1}}\leq \tau$.

Conditioned on the event $\mathcal{E}$ as defined in equation \ref{eq:high-prob-event}, it follows that:
\begin{align*}
    |\langle x_{\pi}^{o, \perp} ,   \widehat{\mu}_t^{o, \perp} -\mu_*^{o, \perp}\rangle| &\leq \| \mu_*^{o, \perp} - \widehat{\mu}_t^{o, \perp} \|_{\Sigma_t^{o, \perp}}\| x_{\pi}\|_{(\Sigma^{o, \perp}_t)^{-1}}   \\
    &\leq \langle x_{\pi}^{o, \perp} ,   \widehat{\mu}_t^{o, \perp} -\mu_*^{o, \perp}\rangle \beta_t
    (\delta, d-1) \| x_{\pi}\|_{(\Sigma^{o, \perp}_t)^{-1}}
\end{align*}
And therefore:
\begin{equation}
    0 \leq \langle x_{\pi}^{o, \perp} ,   \widehat{\mu}_t^{o, \perp} -\mu_*^{o, \perp}\rangle + \beta_t
    (\delta, d-1) \| x_{\pi}\|_{(\Sigma^{o, \perp}_t)^{-1}} \label{equation::prop_3_eq1}
\end{equation}
Observe that:
\begin{align}
c_{\pi} &= \frac{ \langle x_{\pi}^o, e_0\rangle c_0}{ \| x_0\| }  + \langle x_{\pi}^{o, \perp}, \mu_*^{o, \perp}\rangle \notag\\ 
&\leq \underbrace{\frac{ \langle x_{\pi}^o, e_0\rangle c_0}{ \| x_0\| }  + \langle x_{\pi}^{o, \perp}, \widehat{\mu}_{t}^{o, \perp}\rangle  + \alpha_c \beta_t(\delta, d-1)\| x_{\pi}^{o, \perp} \|_{ (\Sigma^{o, \perp}_t)^{-1}}}_{\mathrm{I}}\label{equation::prop_3_eq2}
\end{align}
The last inequality holds by adding Inequality \ref{equation::prop_3_eq1} to Inequality \ref{equation::prop_3_eq2}. Since by assumption for all $\pi \in \Pi_t$ term $I \leq \tau$, we obtain that $c_\pi \leq \tau$. The result follows.
\end{proof}


\newpage
\section{Proofs of Section~\ref{section::lin_opt_pess_analysis}}
\label{sec:proofs-analysis-section}


\subsection{Proof of Lemma~\ref{lemma:bounding-4}}
\label{subsec:proof-Lemma-bounding-4}

We first state the following proposition that is used in the proof of Lemma~\ref{lemma:bounding-4}. This proposition is a direct consequence of Eq.~20.9 and Lemma~19.4 in~\cite{lattimore2018bandit}. Similar result has also been reported in the appendix of~\cite{amani2019linear}.

\begin{proposition}
\label{proposition:det_lemma}
For any sequence of actions $(x_1,\ldots,x_t)$, let $\Sigma_t$ be its corresponding Gram matrix defined by~\eqref{eq:Sigmas} with $\lambda \geq 1$. Then, for all $t\in[T]$, we have

\vspace{-0.15in}
\begin{small}
\begin{equation*}
\sum_{s=1}^T \| x_s \|_{\Sigma^{-1}_{s}} \leq  \sqrt{2Td\log\big( 1+\frac{TL^2}{\lambda}\big)}.
\end{equation*}
\end{small}
\end{proposition}

We now state the proof of Lemma~\ref{lemma:bounding-4}.

\begin{proof}[Proof of Lemma~\ref{lemma:bounding-4}]
We prove this lemma through the following sequence of inequalities:

\vspace{-0.15in}
\begin{small}
\begin{align*}
\sum_{t=1}^T \langle x_t, \widetilde{\theta}_t\rangle - \langle x_t, \theta_* \rangle &\stackrel{\text{(a)}}{\leq} \sum_{t=1}^T \| x_t \|_{\Sigma_t^{-1}}\| \widetilde{\theta}_t - \theta_* \|_{\Sigma_t} \stackrel{\text{(b)}}{\leq} \sum_{t=1}^T (1+\alpha_r) \beta_t(\delta,d)\| x_t \|_{\Sigma_t^{-1}} \\ 
&\stackrel{\text{(c)}}{\leq} (1+\alpha_r)\beta_T(\delta,d) \sum_{t=1}^T  \| x_t \|_{\Sigma_t^{-1}} \stackrel{\text{(d)}}{\leq}(1+\alpha_r)\beta_T(\delta,d)  \sqrt{2Td\log\big(1+\frac{TL^2}{\lambda}\big)}
\end{align*}
\end{small}
\vspace{-0.15in}

{\bf (a)} This is by Cauchy-Schwartz. 

{\bf (b)} This follows from the fact that $\widetilde{\theta}_t\in \mathcal{C}_t^r(\alpha_r)$ and we are on event $\mathcal E$.

{\bf (c)} This is because $\beta_t(\delta,d)$ is an increasing function of $t$, i.e.,~$\beta_T(\delta,d) \geq \beta_t(\delta,d),\;\forall t\in [T]$.

{\bf (d)} This is a direct result of Proposition~\ref{proposition:det_lemma}.
\end{proof}


\subsection{Proof of Lemma~\ref{lemma:inverse_norm_domination}}
\label{subsec:proof-Lemma-inverse-norm-domination}

\begin{proof}

In order to prove the desired result it is enough to show that:
\begin{equation*}
    \left( x_{\pi}^{o, \perp}\right)^\top  \left(  \Sigma_t^{o, \perp}\right)^{\dagger} x_\pi^{o, \perp} \leq x_\pi^\top \Sigma_t^{-1}x_\pi
\end{equation*}

w.l.o.g. we can assume $x_o = e_1$, the first basis vector. Notice that in this case $\Sigma_t^{o, \perp}$ can be thought of as a submatrix of $\Sigma_t$ such that $\Sigma_t[2:, 2:] = \Sigma_t^{o, \perp}$, where $\Sigma_t[2:, 2:]$ denotes the submatrix with row and column indices from $2$ onwards.

Using the following formula for the inverse of a psd symmetric matrix:

\begin{equation*}
    \begin{bmatrix}
    Z & \delta\\
    \delta^\top & A 
    \end{bmatrix} = \begin{bmatrix}
                   \frac{1}{D} &  -\frac{ A^{-1}\delta}{D}\\
                   - \frac{\delta^\top A^{-1} }{D}&A^{-1} + \frac{ A^{1} \delta \delta^\top A^{-1}}{D} 
                \end{bmatrix}
\end{equation*}

Where $D = z- \delta^\top A^{-1} \delta$. In our case $D = \Sigma_t[1,1] - \Sigma_t[2:d]^\top \left(  \Sigma_t^{o, \perp}  \right)^{-1} \Sigma_t[2:d] \in \mathbb{R}$. Observe that since $\Sigma_t$ is PSD, $D \geq 0$. Therefore:
\begin{equation*}
    \Sigma_t^{-1} = \begin{bmatrix}
                    1/D & -\frac{\left(\Sigma_t^{o, \perp}\right)^{-1} \Sigma_t[2, :d] }{D} \\
                    -\frac{\Sigma_t^\top [2:d]\left( \Sigma_t^{o, \perp}\right)^{-1} }{D} & \left( \Sigma_t^{o, \perp}    \right)^{-1} + \frac{ \left( \Sigma_t^{o, \perp}\right)^{-1} \Sigma_t[2:d]\Sigma_t[2:d]\left(\Sigma_t^{o, \perp}\right)^{-1} }{D}
                \end{bmatrix}
\end{equation*}

Then:

\begin{align*}
    x_\pi^\top \left(    \Sigma_t^{-1} \right)^{-1} x_\pi &= \frac{ x_\pi(1)^2 - 2x_\pi(1)\Sigma_t[2:d]^\top\left(   \Sigma_t^{o, \perp}  \right)^{-1} x_\pi[2:d]}{D} +\\
    &\frac{ x_\pi[2:d]^\top \left(\Sigma_t^{o, \perp} \right)^{-1} \Sigma_t[2:d]\Sigma_t[2:d]^\top\left(   \Sigma_t^{o, \perp}  \right)^{-1} x_\pi[2:d] }{D} \\
    &+ x_\pi[2:d]^\top \left( \Sigma_t^{o, \perp}     \right)^{-1} x_\pi[2:d]\\
    &\geq x_\pi[2:d]^\top \left( \Sigma_t^{o, \perp}     \right)^{-1} x_\pi[2:d]
\end{align*}

The result follows by noting that $x_\pi[2:d] = x_\pi^{o, \perp}$.
\end{proof}


\subsection{Proof of Lemma~\ref{lemma:linear_bandits_optimism}}
\label{subsec:proof-Lemma-term-I}

\begin{proof}
For any policy $\pi$, we have 
\begin{equation}
\label{eq:tempB0}
\widetilde{r}_{\pi,t} = \max_{\theta \in \mathcal{C}_t^r(\alpha_r)} \langle x_\pi,\theta \rangle \geq \langle x_\pi, \theta_* \rangle = r_\pi.
\end{equation}
If $\pi_t^* \in \Pi_t$, then by the definition of $\pi_t$ (Line~4 of Algorithm~\ref{alg:optimistic-pessimistic-LB}), we have
\begin{equation}
\label{eq:tempB1}
\widetilde{r}_{\pi_t,t} \geq \widetilde{r}_{\pi_t^*,t}. 
\end{equation}
Combining~\eqref{eq:tempB0} and~\eqref{eq:tempB1}, we may conclude that $\widetilde{r}_{\pi_t,t} \geq r_{\pi_t^*}$ as desired.

We now focus on the case that $\pi_t^* \not\in \Pi_t$, i.e., 
\begin{equation*}
\widetilde{c}_{\pi_t^*,t} = \frac{\langle x_{\pi^*_t}^o,e_0 \rangle c_0}{\|x_0\|} + \langle x_{\pi^*_t}^{o,\perp},\widehat{\mu}^{o,\perp}_t \rangle + \alpha_c\beta_t(\delta,d-1) \| x_{\pi^*_t}^{o,\perp} \|_{(\Sigma_t^{o,\perp})^{-1}} > \tau.
\end{equation*}
We define a mixture policy $\widetilde{\pi}_t = \eta_t \pi_t^* + (1-\eta_t)\pi_0$, where $\pi_0$ is the policy that always selects the safe action $x_0$ and $\eta_t \in [0,1]$ is the maximum value of $\eta$ such that $\big(\eta\pi^*_t+(1-\eta)\pi_0\big) \in \Pi_t$. Conceptually, $\eta_t$ shows how close is the optimal policy $\pi^*_t$ to the set of safe policies $\Pi_t$. 

By the definition of $\widetilde{\pi}_t$, we have
\begin{equation}
\label{eq:tempB2}
x^o_{\widetilde{\pi}_t} = \eta_t x^o_{\pi_t^*} + (1-\eta_t)x_0, \qquad x^{o,\perp}_{\widetilde{\pi}_t} = \eta_t x^{o,\perp}_{\pi_t^*},
\end{equation}
which allows us to write
\begin{align*}
\widetilde{c}_{\widetilde{\pi}_t,t} &= \frac{\eta_t\langle x_{\pi_t^*}^o,e_0 \rangle + (1-\eta_t) \langle x_0,e_0 \rangle}{\|x_0\|}\cdot c_0 + \eta_t \langle x_{\pi_t^*}^{o,\perp}, \widehat{\mu}^{o,\perp}_t \rangle + \eta_t\alpha_c\beta_t(\delta,d-1) \| x_{\pi_t^*}^{o,\perp} \|_{(\Sigma_t^{o,\perp})^{-1} } \\ 
&= \frac{(1-\eta_t) \langle x_0,e_0 \rangle c_0}{\|x_0\|} + \eta_t \widetilde{c}_{\pi_t^*,t}.
\end{align*}
From the definition of $\eta_t$, we have $\widetilde{c}_{\widetilde{\pi}_t,t}=\frac{(1-\eta_t) \langle x_0,e_0 \rangle c_0}{\|x_0\|} + \eta_t \widetilde{c}_{\pi_t^*,t}=\tau$, and thus, we may write
\begin{align}
\eta_t &= \frac{\tau - \frac{\langle x_0,e_0 \rangle c_0}{\|x_0\|}}{\widetilde{c}_{\pi_t^*,t} - \frac{\langle x_0,e_0 \rangle c_0}{\|x_0\|}} = \frac{\tau - c_0}{\frac{\langle x_{\pi_t^*}^o,e_0\rangle c_0}{\|x_0\|} + \langle x_{\pi_t^*}^{o,\perp}, \widehat{\mu}^{o,\perp}_t \rangle + \alpha_c\beta_t(\delta,d-1)\|x^{o,\perp}_{\pi_t^*}\|_{(\Sigma^{0, \perp}_t)^{-1}} - c_0} \nonumber \\
&= \frac{\tau - c_0}{\frac{\langle x_{\pi_t^*}^o,e_0 \rangle c_0}{\|x_0\|} + \langle x_{\pi_t^*}^{o,\perp},\mu_* \rangle + \langle x_{\pi_t^*}^{o, \perp}, \widehat{\mu}^{o,\perp}_t - \mu_* \rangle + \alpha_c\beta_t(\delta, d-1)\|x^{o,\perp}_{\pi_t^*}\|_{(\Sigma^{o, \perp}_t)^{-1}} - c_0} \nonumber \\
&\stackrel{\text{(a)}}{\geq} \frac{\tau - c_0}{\frac{\langle x_{\pi_t^*}^o, e_0 \rangle c_0}{\|x_0\|} + \langle x_{\pi_t^*}^{o,\perp},\mu_* \rangle + (1+\alpha_c)\beta_t(\delta,d-1)\|x^{o,\perp}_{\pi_t^*}\|_{(\Sigma_t^{o,\perp})^{-1}} - c_0} \nonumber \\
&\stackrel{\text{(b)}}{\geq} \frac{\tau - c_0}{\tau + (\alpha_c+1)\beta_t(\delta,d-1)\|x^{o,\perp}_{\pi_t^*}\|_{(\Sigma_t^{o,\perp})^{-1}} - c_0}.
\label{eq:tempB3}
\end{align}
{\bf (a)} This holds because 
\begin{equation*}
\langle x^{o,\perp}_{\pi_t^*},\widehat{\mu}^{o,\perp}_t - \mu_* \rangle = \langle x^{o,\perp}_{\pi_t^*},\widehat{\mu}^{o,\perp}_t - \mu^{o,\perp}_* \rangle \leq \|\widehat{\mu}^{o,\perp}_t - \mu_*^{o,\perp}\|_{\Sigma^{o,\perp}_t} \|x^{o,\perp}_{\pi_t^*}\|_{(\Sigma^{o,\perp}_t)^{-1}} \leq  \beta_t(\delta, d-1) \|x^{o,\perp}_{\pi_t^*}\|_{(\Sigma^{o, \perp}_t)^{-1}}, 
\end{equation*}
where the last inequality is because we are on the event $\mathcal E$. 

{\bf (b)} This passage is due to the fact that the optimal policy $\pi^*_t$ is feasible, and thus, $\mathbb E_{x\sim\pi_t^*}[\langle x,\mu_* \rangle] \leq \tau$. Therefore, we may write 
\begin{align*}
\mathbb E_{x\sim\pi_t^*}[\langle x,\mu_* \rangle] &= \mathbb E_{x\sim\pi_t^*}[\langle x^o,\mu_* \rangle] + \langle x_{\pi_t^*}^{o,\perp},\mu_* \rangle = \mathbb E_{x\sim\pi_t^*}[\langle \langle x,e_0 \rangle e_0,\mu_* \rangle] + \langle x_{\pi_t^*}^{o,\perp},\mu_* \rangle \\
&= \mathbb E_{x\sim\pi_t^*}[\langle \langle x,e_0 \rangle \frac{x_0}{\|x_0\|},\mu_* \rangle] + \langle x_{\pi_t^*}^{o,\perp},\mu_* \rangle = \frac{c_0}{\|x_0\|}\mathbb E_{x\sim\pi_t^*}[\langle x,e_0\rangle] + \langle x_{\pi_t^*}^{o,\perp},\mu_* \rangle \\
&= \frac{\langle x^o_{\pi_t^*},e_0 \rangle c_0}{\|x_0\|} + \langle x_{\pi_t^*}^{o,\perp},\mu_* \rangle \leq \tau.
\end{align*}
Since $\widetilde{\pi}_t \in \Pi_t$, we have
\begin{align}
\widetilde{r}_{\pi_t,t} &\geq \widetilde{r}_{\widetilde{\pi}_t,t} = \langle x_{\widetilde{\pi}_t},\widehat{\theta}_t \rangle + \alpha_r\beta_t(\delta,d)\|x_{\widetilde{\pi}_t}\|_{\Sigma_t^{-1}} = \langle x_{\widetilde{\pi}_t},\theta_* \rangle + \langle x_{\widetilde{\pi}_t},\widehat{\theta}_t - \theta_* \rangle + \alpha_r\beta_t(\delta,d)\|x_{\widetilde{\pi}_t}\|_{\Sigma_t^{-1}} \nonumber \\
&\stackrel{\text{(a)}}{\geq} \langle x_{\widetilde{\pi}_t},\theta_* \rangle +  (\alpha_r - 1)\beta_t(\delta,d)\|x_{\widetilde{\pi}_t}\|_{\Sigma_t^{-1}} \stackrel{\text{(b)}}{\geq} \langle x_{\widetilde{\pi}_t},\theta_* \rangle + (\alpha_r - 1)\beta_t(\delta,d-1)\|x^{o,\perp}_{\widetilde{\pi}_t}\|_{(\Sigma^{o,\perp}_t)^{-1}} \nonumber \\
&\stackrel{\text{(c)}}{=} \eta_t \langle x_{\pi^*},\theta_* \rangle + (1-\eta_t)\langle x_0,\theta_* \rangle + \eta_t(\alpha_r-1)\beta_t(\delta,d-1)\|x^{o,\perp}_{\pi_t^*}\|_{(\Sigma^{o, \perp}_t)^{-1}} \nonumber \\
&\stackrel{\text{(d)}}{\geq} \eta_t \langle x_{\pi_t^*},\theta_* \rangle + \eta_t(\alpha_r-1)\beta_t(\delta,d-1)\|x^{o,\perp}_{\pi^*}\|_{(\Sigma^{o,\perp}_t)^{-1}} \nonumber \\
&\stackrel{\text{(e)}}{\geq} \underbrace{\Big(\frac{\tau - c_0}{\tau - c_0 + (\alpha_c+1) \beta_t(\delta, d-1) \|x^{o,\perp}_{\pi_t^*}\|_{(\Sigma_t^{o,\perp})^{-1}}}\Big) \Big(\langle x_{\pi_t^*},\theta_* \rangle + (\alpha_r-1)\beta_t(\delta,d-1) \|x^{o,\perp}_{\pi_t^*}\|_{(\Sigma^{o,\perp}_t)^{-1}}\Big)}_{C_0}.
\label{eq:tempB4}
\end{align}
{\bf (a)} This is because we may write 
\begin{equation*}
|\langle x_{\widetilde{\pi}_t},\widehat{\theta}_t - \theta_* \rangle| \leq \|\widehat{\theta}_t - \theta_*\|_{\Sigma_t} \|x_{\widetilde{\pi}_t}\|_{\Sigma_t^{-1}} \leq \beta_t(\delta,d) \|x_{\widetilde{\pi}_t}\|_{\Sigma_t^{-1}},
\end{equation*}
where the last inequality is due to the fact that we are on the event $\mathcal E$. Thus, $\langle x_{\widetilde{\pi}_t},\widehat{\theta}_t - \theta_* \rangle \geq -\beta_t(\delta,d) \|x_{\widetilde{\pi}_t}\|_{\Sigma_t^{-1}}$.

{\bf (b)} This is a consequence of Lemma~\ref{lemma:inverse_norm_domination} stated in the paper and proved in Appendix~\ref{subsec:proof-Lemma-inverse-norm-domination}. 

{\bf (c)} This is from the definition of $\widetilde{\pi}$ and Eq.~\ref{eq:tempB2}.

{\bf (d)} This is because $\eta_t\in [0,1]$ and from  Assumption~\ref{ass:bounded-mean-reward-cost} we have that all expected rewards are positive (belong to $[0,1]$), and thus, $\langle x_0,\theta_* \rangle \geq 0$. 

{\bf (e)} This is by lower-bounding $\eta_t$ from~\eqref{eq:tempB3}. 

Let us define the shorthand notation $C_1:=\beta_t(\delta,d-1)\|x^{o,\perp}_{\pi_t^*}\|_{(\Sigma^{o,\perp}_t)^{-1}}$. Thus, we may write $C_0$ as 
\begin{equation*}
C_0 = \frac{\tau - c_0}{\tau - c_0 + (1+\alpha_c)C_1}\times\big(\langle x_{\pi_t^*},\theta_* \rangle + (\alpha_r-1)C_1\big).
\end{equation*}
Note that $C_0 \geq \langle x_{\pi_t^*}, \theta_* \rangle = r_{\pi^*_t}$ (and as a results $\widetilde{r}_{\pi_t,t}\geq r_{\pi^*_t}$ as desired) iff:
\begin{equation*}
(\tau-c_0) r_{\pi_t^*} + (\tau - c_0) (\alpha_r - 1) C_1 \geq (\tau - c_0) r_{\pi_t^*} + (1 + \alpha_c) C_1 r_{\pi_t^*},
\end{equation*}
which holds iff: $(\tau-c_0)(\alpha_r-1)C_1 \geq (1+\alpha_c)C_1r_{\pi_t^*}$.

Since $r_{\pi_t^*}\leq 1$ from Assumption~\ref{ass:bounded-mean-reward-cost}, this holds iff: $1+\alpha_c \leq (\tau-c_0) (\alpha_r-1)$. This concludes the proof as for both cases of $\pi^*_t\in\Pi_t$ and $\pi^*_t\not\in\Pi_t$, we proved that $\widetilde{r}_{\pi_t,t}\geq r_{\pi_t^*}$.
\end{proof}

\subsection{Learning the safe policy's value}\label{section::safe_policy_value}

In this section we relax Assumption~\ref{ass:safe-action}, and instead assume we only have the knowledge of a safe arm, but not any knowledge of its value $c_0$.

If the cost of the safe arm $c_0$ is unknown, we start by taking the safe action $x_0$ for $T_0$ rounds to produce first an empirical mean estimator for $\hat{c}_9$. Notice that for all $\delta \in (0,1)$, $\hat{c}_0$ satisfies:

\begin{equation}\label{equation::helper_unknown_c0}
    \mathbb{P}\left(\hat{c}_0 \leq c_0 - \sqrt{\frac{2 \log\left(1/\delta \right)}{T_0} }       \right) \leq \delta
\end{equation}

Let $\tilde{c}_0 = \hat{c}_0 + \sqrt{\frac{2 \log\left(1/\delta \right)}{T_0} } $. By inequality~\ref{equation::helper_unknown_c0}, it follows that with probability at least $1-\delta$:
\begin{equation*}
\tilde{c}_0 \geq c_0  
\end{equation*}
We select $T_0$ in an adaptive way. In other words, we do the following:

Let $\delta  = \frac{1}{T^2}$. And let $\hat{c}_0(t)$ be the sample mean estimator of $c_0$, when using only $t$ samples. Similarly define $\tilde{c}_0(t) = \hat{c}_0(t) + \sqrt{\frac{2\log(1/\delta)}{t} }$ Let's condition on the event $\mathcal{E}$ that for all $t \in [T]$:
\begin{equation*}
    |\hat{c}_0(t)  - c_0| \leq \sqrt{\frac{ 2\log(1/\delta)}{t} } 
\end{equation*}
By assumption $\mathbb{P}(\mathcal{E}) \geq 1-T2\delta = 1-\frac{2}{T}$. Let $T_0$ be the first time that $\tilde{c}_0(T_0) + 2\sqrt{\frac{ 2\log(1/\delta)}{T_0}} \leq \tau$. 

Notice that in this case and conditioned on $\mathcal{E}$ and therefore on $\tilde{c}_0(T_0) \geq c_0$:

\begin{equation*}
    \sqrt{ \frac{2 \log(1/\delta)}{T_0}} \leq \frac{\tau - c_0}{2} \quad\text{ i.e. }\quad T_0 \geq \frac{8\log(1/\delta)}{(\tau - c_0)^2}
\end{equation*}

In other words, this test does not stop until $T_0 \geq \frac{8\log(1/\delta)}{(\tau - c_0)^2}$. Now we see it won't take much longer than that to stop:

Conversely, let $T_0' \geq \frac{32\log(1/\delta)}{(\tau-  c_0)^2}$. For any such $T_0'$ we observe that by conditioning on $\mathcal{E}$:
\begin{equation*}
    \tilde{c}_0(T_0')  + 2 \sqrt{\frac{ 2\log(1/\delta)}{T_0'} } \leq c_0 + 4 \sqrt{\frac{ 2\log(1/\delta)}{T_0'} }\leq \tau
\end{equation*}
Thus conditioned on $\mathcal{E}$, we conclude $\frac{8\log(1/\delta)}{(\tau - c_0)^2} \leq T_0 \leq \frac{32\log(1/\delta)}{(\tau-  c_0)^2} $. Then, 

Therefore $\hat{\delta}_c = \sqrt{ \frac{8 \log(1/\delta)}{T_0}}$ would serve as a conservative estimator for $\frac{\tau - c_0}{2}$ satisfying:
\begin{equation*}
 \frac{\tau-c_0}{2} \leq \hat{\delta}_c \leq   \tau - c_0  
\end{equation*}

We proceed by warm starting our estimators for $\theta_*$ and $\mu_*$ using the data collected by playing $x_0$. However, instead of estimating $\mu_*^{o, \perp}$, we build an estimator for $\mu_*$ over all its directions, including $e_0$, similar to what OPLB does for $\theta_*$. We then set $\frac{\alpha_r}{\alpha_c} = 1/\hat\delta_c$ and run Algorithm~\ref{alg:optimistic-pessimistic-LB} for rounds $t > T_0$. Since the scaling of $\alpha_r$ w.r.t. $\alpha_c$ is optimal up to constants, the same arguments hold.

%% file: contrained_muli_armed_bandit_appendix.tex
\section{Constrained Multi-Armed Bandits}
\label{section::constrained_multi_armed_bandit_setup}

%% file: optimism_pessimism_MAB_appendix.tex
\subsection{Optimism Pessimism}\label{section::appendix_MAB_optimism_pessimism_algorithm}

Here we reproduce the full pseudo-code for OPB:

{\centering
\begin{minipage}{.9\linewidth}
    \begin{algorithm}[H]
    \textbf{Input:} Number of arms $K$, constants $\alpha_r, \alpha_c \geq 1$. \\
    \For{$t=1, \ldots , T$}{
    1. Compute estimates $\{u_a^r(t) \}_{a\in\mathcal A}$, $\{u_a^c(t)  \}_{a\in\mathcal A}$.  \\
    2. Form the approximate LP (\ref{eq::noisy_LP}) using these estimates.\\
    3. Find policy $\pi_t$ by solving~\eqref{eq::noisy_LP}.\\
    4. Play arm $a \sim \pi_t$ \; \\
     }    
     \caption{Optimism-Pessimism}
    \label{alg::optimism_pessimism}
    \end{algorithm}
\end{minipage}
}

Similar to the case of OPLB, we define $\Pi_t = \{ \pi \in \Delta_{\mathcal{A}} : \sum_{a \in \mathcal{A}} \pi_a u_a^c(t) \leq \tau \}$. We also define $\beta_a(0) = 0$ for all $a\in \mathcal{A}$. 

%% file: LP_structure_appendix.tex
\subsection{The LP Structure}\label{section::LP_structure_appendix}

The main purpose of this section is to prove the optimal solutions of the linear program from (\ref{eq::noisy_LP}) are supported on a set of size at most $2$. This structural result will prove important to develop simple efficient algorithms to solve for solving it. Let's recall the form of the Linear program in \ref{eq::noisy_LP} is:
\begin{align*}
    \max_{\pi \in \Delta_K} \sum_{a \in \mathcal{A}} \pi_a u^r_{a}(t) \\
    \text{s.t. } \sum_{a \in \mathcal{A}} \pi_a u_a^c(t) \leq \tau 
\end{align*}

Let's start by observing that in the case $K= 2$ with $\mathcal{A}  = \{ a_1, a_2\}$ and $u_{a_1}^c(t) < \tau < u_{a_2}^c(t) $, the optimal policy $\pi^*$ is a mixture policy satisfying:
\begin{align}
    \pi_{a_1}^* &= \frac{ u_{a_2}^c(t) - \tau  }{u^c_{a_2}(t) - u^c_{a_1}(t)} \notag\\
    \pi_{a_2}^* &= \frac{ \tau - u_{a_1}^c(t)}{u^c_{a_2}(t) - u^c_{a_1}(t)} \label{equation::optimal_policy_pair}
\end{align}

The main result in this section is the following Lemma:

\begin{lemma}[$\pi^*$ support]\label{lemma::LP_support_appendix} If (\ref{eq::noisy_LP}) is feasible, there exists an optimal solution with at most $2$ non-zero entries. 
\end{lemma}

\begin{proof}
We start by inspecting the dual problem of (\ref{eq::noisy_LP}):
\begin{equation}\label{equation::dual_LP}
    \min_{\lambda \geq 0} \max_{a} \lambda( \tau - u_a^c(t) ) + u_a^r(t) \tag{D}
\end{equation}
This formulation is easily interpretable. The quantity $\tau - u_a^c(t)$ measures the feasibility gap of arm $a$, while $u^r_a(t)$ introduces a dependency on the reward signal. Let $\lambda^*$ be the optimal value of the dual variable $\lambda$. Define $\mathcal{A}^*\subseteq \mathcal{A}$ as $\mathcal{A}^* = \arg\max_a \lambda^* (\tau -  u^c_a(t) ) + u^r_a(t)$. By complementary slackness the set of nonzero entries of $\pi^*$ must be a subset of $\mathcal{A}^*$.

If $\left|\mathcal{A}^* \right| = 1$, complementary slackness immediately implies the desired result. If $a_1, a_2$ are two elements of $\mathcal{A}^*$, it is easy to see that:
\begin{equation*}
  u^r_{a_1}(t) - \lambda^* u^c_{a_1}(t) = u^r_{a_2}(t) -\lambda^* u^c_{a_2}(t) , 
\end{equation*}
and thus,
\begin{equation}
\label{eq::dual_lambda_explicit}
    \lambda^* = \frac{ u_{a_2}^r(t) - u^r_{a_1}(t)}{ u_{a_2}^c(t) - u_{a_1}^c(t)}
\end{equation}
If $\lambda^* = 0$, the optimal primal value is achieved by concentrating all mass on any of the arms in $\mathcal{A}^*$. Otherwise, plugging~\ref{eq::dual_lambda_explicit} back into the objective of (\ref{equation::dual_LP}) and rearranging the terms, we obtain
\begin{align*}
    s
\end{align*}

\begin{align*}
    \text{(D)}&= \lambda^*( \tau- u_{a_1}^c(t)) + u^r_{a_1}(t)\\
    &= u^r_{a_1}(t)\left(\frac{ \tau - u_{a_1}^c(t) }{ u_{a_2}^c(t) - u_{a_1}^c(t)} \right) + u_{a_2}^r(t) \left(\frac{ u_{a_2}^c(t) - \tau  }{ u_{a_2}^c(t) - u_{a_1}^c(t) }  \right).
\end{align*}
If $u_{a_2}^c(t)  \geq \tau \geq u_{a_1}^c(t)$, we obtain a feasible value for the primal variable $\pi_{a_1}^* = \frac{ \tau - u_{a_1}^c(t) }{ u_{a_2}^c(t) - u_{a_1}^c(t)} $, $\pi_{a_2}^* = \frac{  u_{a_2}^c(t) -\tau }{ u_{a_2}^c(t) - u_{a_1}^c(t)}$ and zero for all other $a \in \mathcal{A} \backslash \{ a_1, a_2 \}$. Since we have assumed (\ref{eq::noisy_LP}) to be feasible there must be either one arm $a^* \in \mathcal{A}^*$ satisfying $a^* = \arg\max_{a \in \mathcal{A}^*} u_{a}^r(t)$ and $u_{a^*}^c(t) \leq \tau$ or two such arms $a_1$ and $a_2$ in $\mathcal{A}^*$ that satisfy $u_{a_2}^c(t)  \geq \tau \geq u_{a_1}^c(t)$, since otherwise it would be impossible to produce a feasible primal solution without having any of its supporting arms $a$ satisfying $u_a^c(t) \leq \tau$, there must exist an arm $a \in \mathcal{A}^*$ with $u_a^c(t) < \tau$. This completes the proof. \end{proof}

 From the proof of Lemma \ref{lemma::LP_support} we can conclude the optimal policy is either a delta mass centered at the arm with the largest reward - whenever this arm is feasible -  or it is a strict mixture supported on two arms. 

A further consequence of Lemma~\ref{lemma::LP_support_appendix} is that it is possible to find the optimal solution $\pi^*$ to problem~\ref{eq::noisy_LP} by simply enumerating all pairs of arms $(a_i, a_j)$ and all singletons, compute their optimal policies (if feasible) using Equation~\ref{equation::optimal_policy_pair} and their values and selecting the feasible pair (or singleton) achieving the largest value. More sophisticated methods can be developed by taking into account elimination strategies to prune out arms that can be determined in advance not to be optimal nor to belong to an optimal pair. Overall this method is more efficient than running a linear programming solver on~(\ref{eq::noisy_LP}).

If we had instead $m$ constraints, a similar statement to Lemma \ref{lemma::LP_support} holds, namely it is possible to show the optimal policy will have support of size at most $m+1$. The proof is left as an exercise for the reader.

%% file: regret_analysis_MAB_appendix.tex
\subsection{Regret analysis}\label{section:regret_analysis_appendix}

In order to show a regret bound for Algorithm \ref{alg::optimism_pessimism}, we start with the following regret decomposition:
\begin{align*}
    \mathcal{R}_\Pi(T) &= \sum_{t=1}^T \mathbb{E}_{a\sim \pi^*}[\bar{r}_a] - \mathbb{E}_{a \sim \pi_t}[\bar{r}_{a}] \\
    &=\underbrace{\left( \sum_{t=1}^T \mathbb{E}_{a\sim \pi^*}[\bar{r}_a] - \mathbb{E}_{a \sim \pi_t}[u_a^r(t)] \right)}_{(i)} +\underbrace{ \left(\sum_{t=1}^T \mathbb{E}_{a\sim \pi_t}[u^r_a(t)] - \mathbb{E}_{a \sim \pi_t}[\bar{r}_{a}] \right) }_{(ii)}.
\end{align*}
In order to bound $\mathcal{R}_\Pi(T)$, we independently bound terms $(i)$ and $(ii)$. 

We start by bounding term (i). We proceed by first proving an Lemma~\ref{lemma::optimism}, the equivalent version of Lemma~\ref{lemma:linear_bandits_optimism} for the multi armed bandit problem.

\subsection{Proof of Lemma~\ref{lemma::optimism}}
\begin{proof} 
Throughout this proof we denote as $\pi_0$ to the delta function over the safe arm $1$. We start by noting that under $\mathcal{E}$, and because $\alpha_r, \alpha_c \geq 1$, then:
\begin{equation}
\label{equation::confidence_interval_lower_bounds}
(\alpha_r-1)\beta_a(t) \leq \xi_a^r(t) \leq (\alpha_r + 1)\beta_a(t) \text{ } \forall a \quad \text{ and } \quad (\alpha_c-1)\beta_a(t) \leq \xi_a^c(t) \leq (\alpha_c+1)\beta_a(t) \text{ } \forall a \neq 0.
\end{equation}
If $\pi^*\in \Pi_t$, it immediately follows that:
\begin{equation}
\label{equation::feasible_lower_bound}
\mathbb{E}_{a\sim \pi^*}\left[\bar{r}_a \right] \leq \mathbb{E}_{a\sim \pi^*}\left[u_a^r(t)\right] \leq \mathbb{E}_{a \sim \pi_t}\left[  u_a^r(t)\right].
\end{equation}

Let's now assume $\pi^* \not\in \Pi_t$, i.e.,~$\mathbb{E}_{a\sim\pi^*}\left[u_a^c(t)\right] > \tau$. Let $\pi^* = \rho^* \bar{\pi}^* + (1-\rho)\pi_0$ with $\bar{\pi}^* \in \Delta_K[2:K]$\footnote{In other words, the support of $\bar{\pi}^*$ does not contain the safe arm $1$.}.

Consider a mixture policy $\widetilde{\pi}_t = \gamma_t \pi^* + (1-\gamma_t)\pi_0 = \gamma_t \rho^* \bar{\pi}^* +  (1-\gamma_t \rho^*)\pi_0$, where $\gamma_t$ is the maximum $\gamma_t\in [0,1]$ such that $\widetilde{\pi}_t\in\Pi_t$. It can be easily established that 
\begin{align*}
\gamma_t &= \frac{\tau-\bar{c}_1}{\rho^*\mathbb{E}_{a \sim \bar{\pi}^*}\left[u^c_a(t)\right] - \rho^*\bar{c}_1} =\frac{\tau - \bar{c}_1}{ \mathbb{E}_{a \sim \bar{\pi}^*}[\rho^*(\bar{c}_a + \xi_a^c(t) )] - \rho^* \bar{c}_1} \\
&\stackrel{(i)}{ \geq}  \frac{\tau - \bar{c}_1}{\tau - \bar{c}_1 + \rho^* (1+\alpha_c)\mathbb{E}_{a \sim \bar{\pi}^*}[\beta_a(t)]}. 
\end{align*}
{\bf (i)} is a consequence of~\eqref{equation::confidence_interval_lower_bounds} and of the observation that since $\pi^*$ is feasible $\rho^*\mathbb{E}_{a \sim \bar{\pi}^*}[\bar{c}_a] + (1-\rho^*)\bar{c}_1 \leq \tau$. Since $\widetilde{\pi}_t\in\Pi_t$, we have
\begin{align*}
    \mathbb{E}_{a \sim \pi_t }[u_a^r(t)] &\geq \underbrace{\gamma_t \mathbb{E}_{a \sim \pi^*}[u_a^r(t)] + (1-\gamma_t)u^r_0(t)}_{\mathbb{E}_{a\sim \widetilde{\pi}_t}\left[ u_a^r(t)\right]} \\ 
    &\stackrel{(ii)}{\geq} \frac{\tau - \bar{c}_1}{\tau - \bar{c}_1 + \rho^* (1+\alpha_c)\mathbb{E}_{a \sim \bar{\pi}^*}[\beta_a(t)]} \times \mathbb{E}_{a \sim \pi^*}[u_a^r(t)] \\
    &= \frac{\tau-\bar{c}_1}{\tau -\bar{c}_1 + \rho^*(1+\alpha_c)\mathbb{E}_{a \sim \bar{\pi}^*}[\beta_a(t)]} \times \Big(\mathbb{E}_{a \sim \pi^*}[\bar{r}_a] + \mathbb{E}_{a \sim \pi^*}[\xi_a^r(t)]\Big) \\
    &\stackrel{(iii)}{\geq} \frac{\tau - \bar{c}_1}{\tau -\bar{c}_1 +\rho^* (1+\alpha_c)\mathbb{E}_{a \sim \bar{\pi}^*}[\beta_a(t)]} \times \Big(\mathbb{E}_{a \sim \pi^*}[\bar{r}_a] + (\alpha_r-1)\mathbb{E}_{a \sim \pi^*}[\beta_a(t)]\Big) \\
    &\stackrel{(iv)}{\geq}\underbrace{\frac{\tau-\bar{c}_1}{\tau -\bar{c}_1+ (1+\alpha_c)\mathbb{E}_{a \sim \pi^*}[\beta_a(t)]} \times \Big(\mathbb{E}_{a \sim \pi^*}[\bar{r}_a ] + (\alpha_r-1)\mathbb{E}_{a \sim \pi^*}[\beta_a(t)]\Big)}_{C_0}.
\end{align*}
{\bf (ii)} holds because $u_0^r(t)\geq 0$. {\bf (iii)} is a consequence of~\eqref{equation::confidence_interval_lower_bounds} and {\bf (iv)} follows because $\mathbb{E}_{a \sim \pi^*}[ \beta_a(t)] = \rho^* \mathbb{E}_{a \sim \bar{\pi}^*}[\beta_a(t)] + (1-\rho^*)\beta_0(t) \geq \rho^* \mathbb{E}_{a \sim \bar{\pi}^*}[\beta_a(t)]$ since $\beta_a(t) \geq 0$ for all $a$ and $t$. 

Let $C_1 = \mathbb{E}_{a \sim \pi^*}[\beta_a(t)]$. The following holds:
\begin{equation*}
    C_0 = \frac{\tau-\bar{c}_1}{\tau - \bar{c}_1 + (1+\alpha_c)C_1}\times\Big(\mathbb{E}_{a \sim \pi^*}[\bar{r}_a] + (\alpha_r-1)C_1\Big).
\end{equation*}
Note that $C_0 \geq \mathbb{E}_{a \sim \pi^*}\left[\bar{r}_a \right]$ iff:
\begin{equation*}
    (\tau-\bar{c}_1) \mathbb{E}_{a \sim \pi^*}[\bar{r}_a] + (\tau-\bar{c}_1)(\alpha_r-1)C_1 \geq (\tau-\bar{c}_1)\mathbb{E}_{a \sim \pi^*}\left[\bar{r}_a \right] + (1+\alpha_c)C_1\mathbb{E}_{a \sim \pi^*}\left[\bar{r}_a \right],
\end{equation*}
which holds iff:
\begin{equation*}
     (\tau-\bar{c}_1)(\alpha_r-1)C_1 \geq (1+\alpha_c)C_1\mathbb{E}_{a\sim \pi^*}[\bar{r}_a].
 \end{equation*}
Since $\mathbb{E}_{a\sim \pi^*}\left[ \bar{r}_a\right] \leq 1$, this holds if $1+\alpha_c \leq (\tau-\bar{c}_1) (\alpha_r-1)$. 
\end{proof}

\begin{proposition}\label{proposition::bounding_term_I}
If $\delta = \frac{\epsilon}{4KT}$ for $\epsilon \in (0,1)$, $\alpha_r, \alpha_c \geq 1$ with $\alpha_c \leq \tau (\alpha_r -1)$, then with probability at least $1-\frac{\epsilon}{2}$, we have
\begin{equation*}
  \sum_{t=1}^T \mathbb{E}_{a\sim \pi^*}[\bar{r}_a] - \mathbb{E}_{a \sim \pi_t}[u_a^r(t)] \leq 0
\end{equation*}
\end{proposition}

\begin{proof}
A simple union bound implies that $\mathbb{P}(\mathcal{E}) \geq 1-\frac{\epsilon}{2}$. Combining this observation with Lemma \ref{lemma::optimism} yields the result.
\end{proof}

Term $(ii)$ can be bound using the confidence intervals radii:

\begin{proposition}
\label{proposition::bounding_term_II}
If $\delta = \frac{\epsilon}{4KT}$ for an $\epsilon \in (0,1)$, then with probability at least $1-\frac{\epsilon}{2}$, we have
\begin{equation*}
\sum_{t=1}^T \mathbb{E}_{a\sim \pi_t}[u_a^r(t)] - \mathbb{E}_{a \sim \pi_t}[\bar{r}_{a}] \leq (\alpha_r+1) \left(2\sqrt{2TK\log(1/\delta)} + 4\sqrt{T\log(2/\epsilon)\log(1/\delta)} \right)
\end{equation*}
\end{proposition}

\begin{proof}
Under these conditions $\mathbb{P}(\mathcal{E}) \geq 1-
\frac{\epsilon}{2}$. Recall $u_a^r(t) = \widehat{r}_a(t) + \alpha_r\beta_a(t)$ and that conditional on $\mathcal{E}$, $\bar{r}_a \in [\widehat{r}_a(t) - \beta_a(t),\widehat{r}_a(t) + \beta_a(t)]$ for all $t \in [T]$ and $a \in \mathcal{A}$. Thus, for all $t$, we have
\begin{equation*}
    \mathbb{E}_{a \sim \pi_t}[u_a^r(t)] - \mathbb{E}_{a \sim \pi_t}[\bar{r}_a] \leq (\alpha_r +1)\mathbb{E}_{a \sim \pi_t}[\beta_a(t)].
\end{equation*}
Let $\mathcal{F}_{t-1}$ be the sigma algebra defined up to the choice of $\pi_t$ and $a_t'$ be a random variable distributed as $\pi_t \mid \mathcal{F}_{t-1}$ and conditionally independent from $a_t$, i.e.,~$a'_t \perp a_t \mid \mathcal{F}_{t-1}$. Note that by definition the following equality holds: 
\begin{equation*}
    \mathbb{E}_{a \sim \pi_t}[\beta_a(t)] = \mathbb{E}_{a'_t \sim \pi_t}[\beta_a(t) \mid \mathcal{F}_{t-1}].
\end{equation*}
%
Consider the following random variables $A_t =  \mathbb{E}_{a'_t \sim \pi_t} [\beta_{a'_t}(t) \mid \mathcal{F}_{t-1}]- \beta_{a_t}(t)$. Note that $M_t = \sum_{i=1}^t A_i$ is a martingale. Since $|A_t| \leq 2\sqrt{2 \log(1/\delta)}$, a simple application of Azuma-Hoeffding\footnote{We use the following version of Azuma-Hoeffding: if $X_n$, $n\geq 1$ is a martingale such that $|X_i - X_{i-1}| \leq d_i$, for $1 \leq i \leq n$, then for every $n \geq 1$, we have $\mathbb{P}(X_n > r) \leq \exp\left(-\frac{r^2 }{2\sum_{i=1}^n d_i^2}\right)$.} implies:
\begin{equation*}
\mathbb{P}\left(\underbrace{\sum_{t=1}^T \mathbb{E}_{a \sim \pi_t} [\beta_a(t)] \geq \sum_{t=1}^T \beta_{a_t}(t) + 4\sqrt{T\log(2/\epsilon)\log(1/\delta)}}_{\mathcal{E}_A^c}\right ) \leq \epsilon/2.
\end{equation*}
We can now upper-bound $\sum_{t=1}^T \beta_{a_t}(t)$. Note that $\sum_{t=1}^T \beta_{a_t}(t) = \sum_{a \in \mathcal{A}}\sum_{t=1}^T \mathbf{1}\{a_t=a\}\beta_a(t)$. We start by bounding for an action $a\in\mathcal A$:
\begin{align*}
    \sum_{t=1}^T \mathbf{1}\{a_t=a\}\beta_a(t) = \sqrt{2\log(1/\delta)} \sum_{t=1}^{T_a(T)} \frac{1}{\sqrt{t}} \leq 2\sqrt{2T_a(T)\log(1/\delta)}.
\end{align*}
Since $\sum_{a\in\mathcal A } T_a(T) = T$ and by concavity of $\sqrt{ \cdot}$, we have
\begin{equation*}
    \sum_{a\in\mathcal A} 2\sqrt{2T_a(T)\log(1/\delta)} \leq 2\sqrt{2TK\log(1/\delta)}.
\end{equation*}
Conditioning on the event $\mathcal{E} \cap \mathcal{E}_A$ whose probability satisfies $\mathbb{P}( \mathcal{E}\cap \mathcal{E}_A) \geq 1-\epsilon $ yields the result.
\end{proof}

We can combine these two results into our main theorem:

\begin{theorem}[Main Theorem]
If $\epsilon \in(0,1)$, $ \alpha_c=1$ and $\alpha_r = \frac{2}{\tau-\bar{c}_1} + 1$, then with probability at least $1-\epsilon$, Algorithm \ref{alg::optimism_pessimism} satisfies the following regret guarantee:
\begin{equation*}
    \mathcal{R}_\Pi(T)  \leq \left(\frac{2}{\tau-\bar{c}_1} +1\right)\left(2\sqrt{2TK\log(4KT/\epsilon)} + 4\sqrt{T\log(2/\epsilon)\log(4KT/\epsilon)} \right)
\end{equation*}
\end{theorem}

\begin{proof}
This result is a direct consequence of Propositions \ref{proposition::bounding_term_I} and \ref{proposition::bounding_term_II} by setting  $ \delta = 4KT \epsilon$.
\end{proof}

\subsection{Multiple constraints}\label{section::multiple_constraints_appendix}

We consider the problem where the learner must satisfy $M$ constraints with threshold values $\tau_1, \cdots, \tau_M$. Borrowing from the notation in the previous sections, we denote by as $\{\bar{r}_a\}_{a\in \mathcal{A}}$ the mean reward signals and $\{ \bar{c}_a^{(i)} \}$ the mean cost signals for $i = 1,\cdots, M$. The full information optimal policy can be obtained by solving the following linear program:
\begin{align*}
\label{eq::no_noise_LP_multiple_constraints}\tag{P-M}
\max_{\pi \in \mathrm{\Delta}_K} &\; \sum_{a \in \mathcal{A}} \pi_a \bar{r}_a, \\
\text{s.t. } &\; \sum_{a \in \mathcal{A}} \pi_a \bar{c}^{(i)}_a \leq \tau_i \text{ for } i=1, \cdots, M. 
\end{align*}
In order to ensure the learner's ability to produce a feasible policy at all times, we make the following assumption:
\begin{assumption}
    The learner has knowledge of $\bar{c}_1^{(i)} < \tau_i$ for all $i = 1, \cdots, M$.  
\end{assumption}
We denote by $\{ \widehat{r}_a \}_{a \in \mathcal{A}}$ and $\{ \widehat{c}_a^{(i)} \}_{a\in\mathcal{A}}$ for $i = 1, \cdots, M$ the empirical means of the reward and cost signals. We call $\{ u_a^r(t)\}_{a \in \mathcal{A}}$ to the upper confidence bounds for our reward signal and $\{ u_a^{c}(t, i)\}_{a \in \mathcal{A}}$ for $i = 1, \cdots, M$ the costs' upper confidence bounds:
\begin{equation*}
u_a^r(t) = \widehat{r}_a(t) + \alpha_r  \beta_a(t), \qquad u_a^c(t, i) = \widehat{c}^{(i)}_a(t) + \alpha_c  \beta_a(t),
\end{equation*}
where $\beta_a(t) = \sqrt{2\log(1/\delta)/T_a(t)}$, $\delta \in (0,1)$ as before. A straightforward extension of Algorithm \ref{alg::optimism_pessimism} considers instead the following $M-$constraints LP:
\begin{align*}
\label{eq::noisy_LP_multiple}
\tag{$\widehat{P-M}$} 
&\max_{\pi\in\mathrm{\Delta}_K} \;\; \sum_{a \in \mathcal{A}} \pi_a \; u^r_a(t)\;\;\; \\
&\;\text{s.t.} \quad\; \sum_{a \in \mathcal{A}} \pi_a \; u_a^c(t, i)\leq \tau_i, \text{ for } i= 1, \cdots, M.
\end{align*}

We now generalize Lemma \ref{lemma::optimism}:
\begin{lemma}\label{lemma::optimisim_multiple}

Let $\alpha_r, \alpha_c\geq 1$ satisfying $\alpha_c \leq \min_i(\tau_i- \bar{c}^{(i)}_1) (\alpha_r-1)$. Conditioning on $\mathcal{E}_a(t)$ ensures that with probability $1-\delta$:
\begin{equation*}
    \mathbb{E}_{a \sim \pi_t}\left[u_a^r(t)\right] \geq \mathbb{E}_{a \sim \pi^*}\left[\bar{r}_a \right].
\end{equation*}

\end{lemma}

\begin{proof}
The same argument as in the proof of Lemma \ref{lemma::optimism} follows through, the main ingredient is to realize that $\gamma_t$ satisfies the sequence of inequalities in the lemma with $\tau - \bar{c}_1$ substituted by $\min \tau_i - \bar{c}_1^{(i)}$.
\end{proof}

The following result follows:

\begin{theorem}[Multiple Constraints Main Theorem]
If $\epsilon \in(0,1)$, $ \alpha_c=1$ and $\alpha_r = \frac{2}{\min_i \tau_i-\bar{c}^{(i)}_1} + 1$, then with probability at least $1-\epsilon$, Algorithm \ref{alg::optimism_pessimism} satisfies the following regret guarantee:
\begin{equation*}
    \mathcal{R}_\Pi(T)  \leq \left(\frac{2}{\min_i\tau_i-\bar{c}^{(i)}_1} +1\right)\left(2\sqrt{2TK\log(4KT/\epsilon)} + 4\sqrt{T\log(2/\epsilon)\log(4KT/\epsilon)} \right)
\end{equation*}
\end{theorem}

\begin{proof}
The proof follows the exact same argument we used for the proof of Theorem~\ref{theorem::contrained_MAB} substituting $\tau - \bar{c}_1$ by $\min_i \tau_i -\bar{c}_1^{(i)}$.
\end{proof}

\subsection{Lower bound}
\label{section::lower_bound_MAB_appendix}

\input{lower_bounds}

%% file: lower_bounds.tex
We start by proving a generalized version of the divergence decomposition lemma for bandits. 

\begin{lemma}\label{lemma::divergence_decomposition}[Divergence decomposition for constrained multi armed bandits] Let $\nu = ((P_1, Q_1), \cdots, (P_K, Q_K))$ be the reward and constraint distributions associated with one instance of the single constraint multi-armed bandit, and let $\nu' = ((P_1', Q_1'), \cdots, (P_K', Q_K'))$ be the reward and constraint distributions associated with another constrained bandit instance. Fix some algorithm $\mathcal{A}$ and let $\mathbb{P}_\nu = \mathbb{P}_{\nu^\mathcal{A}}$ and $\mathbb{P}_{\nu'} = \mathbb{P}_{\nu'^\mathcal{A}}$ be the probability measures on the cannonical bandit model (See section 4.6 of \cite{lattimore2018bandit}) induced by the $T$ round interconnection of $\mathcal{A}$ and $\nu$ (respectively $\mathcal{A}$ and $\nu'$). Then:
\begin{equation*}
    \mathrm{KL}(\mathbb{P}_\nu, \mathbb{P}_{\nu'}) = \sum_{a=1}^K \mathbb{E}_{\nu}[T_a(T)] \mathrm{KL}((P_a,Q_a), (P_a', Q_a'))
\end{equation*}
Where $T_a(T)$ denotes the number of times arm $a$ was pulled until by $\mathcal{A}$ and up to time $T$. 
\end{lemma}

\begin{proof}
The same proof as in Lemma 15.1 from \cite{lattimore2018bandit} applies in this case. 
\end{proof}

The following two lemmas will prove useful as well:

\begin{lemma}\label{lemma::gaussian_divergence}[Gaussian Divergence ] The divergence between two multivariate normal distributions and means $\mu_1, \mu_2 \in \mathbb{R}^d$ with spherical identity covariance $\mathbb{I}_d$ equals:
\begin{equation*}
    \mathrm{KL}(\mathcal{N}(\mu_1, \mathbb{I}_d) , \mathcal{N}(\mu_2, \mathbb{I}_d) ) = \frac{\| \mu_1- \mu_2\|^2 }{2}
\end{equation*}
\end{lemma}

Define the binary relative entropy to be:
\begin{equation*}
    d(x,y) = x\log(\frac{x}{y}) + (1-x) \log( \frac{1-x}{1-y})
\end{equation*}
and satisfies:
\begin{equation}\label{equation::lower_bound_divergence}
d(x,y) \geq (1/2)\log(1/4y)
\end{equation}
for $x \in [1/2,1]$ and $y \in (0,1)$.
Adapted from~\cite{kaufmann2016complexity}, Lemma 1.
\begin{lemma}\label{lemma::binary_relative_entropy}
Let $\nu, \nu'$ be two constrained bandit models with $K$ arms. Borrow the setup, definitions and notations of Lemma~\ref{lemma::divergence_decomposition}, then for any measurable event $\mathcal{B} \in \mathcal{F}_T$:
\begin{equation}\label{equation::relative_entropy_bound}
   \mathrm{KL}(\mathbb{P}_\nu, \mathbb{P}_{\nu'}) =\sum_{a=1}^K \mathbb{E}_{\nu}[T_a(T)] \mathrm{KL}((P_a,Q_a), (P_a', Q_a')) \geq d(\mathbb{P}_\nu( \mathcal{B})  , \mathbb{P}_{\nu'}(\mathcal{B} ))
\end{equation}
\end{lemma}


We now present a worst-case lower bound for the constrained multi armed bandit problem. We restrict ourselves to Gaussian instances with mean reward and cost vectors $\bar{r}, \bar{c} \in [0,1]^K$. Let $\mathcal{A}$ be an algorithm for policy selection in the constrained MAB problem.  For the purpose of this section
we denote as $\mathcal{R}_\Pi(T, \mathcal{A}, \bar{r}, \bar{c} )$ as the constrained regret of algorithm $\mathcal{A}$ in the Gaussian instance $\mathcal{N}(\bar{r}, \mathbb{I})$, $\mathcal{N}( \bar{c}, \mathbb{I})$. The following theorem holds:

\begin{theorem}
Let $\tau, \bar{c}_1 \in (0,1)$, $K \geq 4$, and $B:=\max\left(\frac{1}{27}\sqrt{ (k-1) T}, \frac{1}{6(\tau - \bar{c}_1)^2} \right)$  and assume\footnote{This constraint on $T$ translates to $T \geq C$ for some constant $C$.} $T \geq \max(  K-1,24eB)$ and let $\tau$ be the maximum allowed cost. Then for any algorithm $\mathcal{A}$ there is a pair of mean vectors $\bar{r}, \bar{c} \in [0,1]^K$ such that:
\begin{equation*}
    \mathcal{R}_\Pi(T, \mathcal{A}, \bar{r}, \bar{c}) \geq B
\end{equation*}
\end{theorem}
\begin{proof}

If $\max\left(\frac{1}{27}\sqrt{ (k-1) T}, \frac{1}{6(\tau - \bar{c}_1)^2} \right) = \sqrt{KT}$, then the argument in Theorem 15.2 of \cite{lattimore2018bandit} yields the desired result by noting that the framework of constrained bandits subsumes unconstrained multi armed bandits when all costs equal zero. In this case we conclude there is an instance $\bar{r}, \bar{c}$ with $\bar{c}_a =0$ for all $a \in \mathcal{A}$ satisfying:
\begin{equation*}
    \mathcal{R}_\Pi(T, \mathcal{A}, \bar{r}, \bar{c}) \geq \frac{1}{27}\sqrt{ (k-1) T}
\end{equation*}

Let's instead focus on the case where $B = \max\left(\frac{1}{27}\sqrt{ (k-1) T}, \frac{1}{6(\tau - \bar{c}_1)^2} \right) =\frac{1}{6(\tau - \bar{c}_1)^2} $.

Pick any algorithm. We want to show that the algorithm's regret on some environment is as large as $B$. If there was an instance $\bar{r}, \bar{c}$ such that $\mathcal{R}_\Pi(T, \mathcal{A}, \bar{r}, \bar{c}) > B$ there would be nothing to be proven. Hence without loss of generality, we can assume that the algorithm satisfies $\mathcal{R}_\Pi(T, \mathcal{A}, \bar{r}, \bar{c}) \leq B $ for all $\bar{r}, \bar{c} \in [0,1]^K$ and having unit variance Gaussian rewards. 

Let $c \in (0,1)$ with $c =  \tau-\bar{c}_1$. For the reader's convenience we will use the notation $\Delta = 1/2$. By treating the rewards in a symbolic way it is easier to understand the logic of the proof argument. Let's consider the following constrained bandit instance inducing measure $\nu$:
\begin{align*}
    \bar{c}^1  &= (\tau - c, &\tau+2c,&\quad \tau-c,& \tau+2c, &\quad\cdots,& \tau+2c )\\
    \bar{r}^1 &= (\Delta,& 8\Delta,&\quad 0, &4\Delta,&\quad \cdots, &4\Delta)
\end{align*}

Notice that the optimal policy equals a mixture between arm $1$ and $2$, where arm $1$ is chosen with probability $2/3$ and arm 2 with probability $1/3$. The value of this optimal policy equals $10/3\Delta$. 

Recall we use the notation $\bar{T}_j(t)$ denote the total amount of probability mass that $\mathcal{A}$ allocated to arm $j$ up to time $t$. Notice that the expected reward of all feasible policies that do not have arm $1$ in their support have a gap (w.r.t the optimal feasible policy's expected reward) of at least $\frac{2\Delta}{3}$. Since by assumption, $\mathcal{A}$ satisfies $\mathcal{R}_\Pi(T, \mathcal{A}, \bar{r}^1, \bar{c}^1) \leq B $:
\begin{align*}
   B &\geq \mathcal{R}_\Pi(T, \mathcal{A}, \bar{r}^1,\bar{c}^1) &\geq\frac{2\Delta}{3} \left( \frac{2}{3}T - \frac{1}{2}T \right) \mathbb{P}\left( \bar{T}_1(T) < \frac{T}{2}  \right)   & = \frac{\Delta}{9}T \mathbb{P}\left( \bar{T}_1(T) < \frac{T}{2}  \right)
\end{align*}
And therefore:
\begin{align*}
    \mathbb{P}\left( \bar{T}_1(T) \geq \frac{T}{2} \right) = 1-\mathbb{P}\left(  \bar{T}_1(T) < \frac{T}{2}  \right) \geq 1-\frac{9B}{\Delta T} \geq 1/2
\end{align*}
The last inequality follows from the assumption $T \geq \max(  K-1,24eB)$.



Let's now consider the following constrained bandit instance inducing measure $\nu'$:
\begin{align*}
    \bar{c}_2 &= (\tau - c, &\tau+2c,&\quad 0,& \tau-c, &\quad\cdots,& \tau+2c )\\
    \bar{r}_2 &= (\Delta,& 8\Delta,&\quad 0, &4\Delta,&\quad \cdots, &4\Delta)
\end{align*}
In this instance the optimal policy is to play arm 4 deterministically, which gets a reward of $4\Delta$. Notice that the expected reward of any feasible policy that does not contain arm $4$ in its support has a gap (w.r.t. the optimal feasible policy's expected reward) of at least $\frac{2\Delta }{3}$. Since by assumption, $\mathcal{A}$ satisfies $\mathcal{R}_\Pi(T, \mathcal{A}, \bar{r}^2, \bar{c}^2) \leq B $:
\begin{align*}
    B &\geq \mathcal{R}_\Pi(T, \mathcal{A}, \bar{r}^2, \bar{c}^2) &\geq\frac{2\Delta}{3} \left(  \frac{1}{2}T \right) \mathbb{P}\left( \bar{T}_1(T) \geq \frac{T}{2}  \right)   & = \frac{\Delta}{3}T \mathbb{P}\left( \bar{T}_1(T) \geq \frac{T}{2}  \right)
\end{align*}
And therefore:
\begin{equation*}
    \mathbb{P}\left( \bar{T}_1(T) \geq \frac{T}{2} \right) \leq \frac{3B}{\Delta T} \leq \frac{1}{4e}
\end{equation*}
The last inequality follows from the assumption $T \geq \max(  K-1,24eB)$. As a consequence of inequality~\ref{equation::lower_bound_divergence}, Lemma~\ref{lemma::binary_relative_entropy} and~\ref{lemma::gaussian_divergence}:
\begin{equation*}
    \mathbb{E}_\nu[T_4(T)]\mathrm{KL}(\binom{\tau+2c}{4\Delta}, \mathbb{I}_d) , \mathcal{N}(\binom{\tau-c}{4\Delta}, \mathbb{I}_d) ) = \mathbb{E}_\nu[T_4(T)]2c^2 \geq \frac{1}{2}
\end{equation*}
And therefore we can conclude:
\begin{equation}
 \mathbb{E}[\bar{T}_4(T)]= \mathbb{E}[T_4(T)]  \geq \frac{1}{4c^2}
\end{equation}
Since in $\nu$, any feasible policy with support in arm $4$ and no support in arm 2 has a suboptimality gap of $4/3\Delta$, we conclude the regret $\mathcal{R}_\Pi(T, \mathcal{A}, \bar{r}^2, \bar{c}^2)$ must satisfy:
\begin{equation*}
    \mathcal{R}_\Pi(T, \mathcal{A}, \bar{r}^2, \bar{c}^2) \geq \frac{\Delta}{3c^2} 
\end{equation*}
Since $\Delta= \frac{1}{2}$ and noting that in this case $\frac{\Delta}{3c^2} = B$. The result follows. 
\end{proof}